\documentclass[11pt]{article}
\usepackage{times}
\usepackage{amsmath}
\usepackage{amsthm}
\usepackage{amssymb}
\usepackage{algorithm}
\usepackage{color}
\usepackage[english]{babel}

\usepackage{graphicx}
\usepackage{caption}
\usepackage{subcaption}

\usepackage{enumitem}

\usepackage{natbib}

\usepackage{wrapfig,epsfig}
\usepackage{psfrag}
\usepackage{epstopdf}
\usepackage{url}
\usepackage{graphicx}
\usepackage{color,xcolor}
\usepackage{epstopdf}
\usepackage{algpseudocode}
\usepackage[hidelinks,pdfencoding=auto,psdextra]{hyperref}
\usepackage{hyperref}
\hypersetup{colorlinks=false}
\hypersetup{
	colorlinks,
	linkcolor={red!40!gray},
	citecolor={blue!40!gray},
	urlcolor={blue!70!gray}
}

\usepackage[margin=1in]{geometry}
\newtheorem{theorem}{Theorem}[section]
\newtheorem{lemma}[theorem]{Lemma}

\newtheorem{definition}[theorem]{Definition}

\newtheorem{corollary}[theorem]{Corollary}

\newtheorem{fact}[theorem]{Fact}
\newtheorem{remark}[theorem]{Remark}
\newtheorem{claim}[theorem]{Claim}

\newcommand{\wt}{\widetilde}
\newcommand{\ov}{\overline}

\newcommand{\eps}{\varepsilon}

\newcommand{\supp}{\mathsf{supp}}
\newcommand{\diag}{\mathsf{diag}}
\newcommand{\sign}{\mathsf{sign}}

\newcommand{\poiss}{\mathsf{Poisson}}

\DeclareMathOperator*{\argmin}{arg\,min}
\DeclareMathOperator*{\E}{\mathbb{E}}

\newcommand{\R}{\mathbb{R}}

\newcommand{\Mc}{\mathcal{M}}

\newenvironment{proofof}[1]{\bigskip \noindent {\bfseries \upshape Proof of #1.}\quad }
{\qed\par\vskip 4mm\par}

\def\cuc{{robust uniform convergence}}
\def\Cuc{{Robust uniform convergence}}

\begin{document}

\title{Query Complexity of
  Least Absolute Deviation Regression\\
  via Robust Uniform Convergence}




\author{\textbf{Xue Chen} \\
George Mason University\\
\texttt{xuechen@gmu.edu}
\and \textbf{Micha{\l} Derezi\'nski}\\
University of California, Berkeley\\
 \texttt{mderezin@berkeley.edu}
}

\maketitle

\date{}


\begin{abstract}%
  Consider a regression problem where the learner is given
  a large collection of $d$-dimensional data points, but can only
  query a small subset of the real-valued labels. How many queries are
  needed to obtain a $1+\epsilon$ relative error approximation of the
  optimum? While this problem has been extensively studied for least
  squares regression, little is known for other losses. An important
  example is least absolute deviation regression ($\ell_1$ regression)
  which enjoys superior robustness to outliers compared to least
  squares. We develop a new framework for analyzing importance
  sampling methods in regression problems, which enables us to show
  that the query complexity of 
  least absolute deviation regression is $\Theta(d/\epsilon^2)$ up to
  logarithmic factors. We further extend our techniques to show the
  first bounds on the query complexity for any $\ell_p$ loss with
  $p\in(1,2)$. As a key novelty in our analysis, we introduce the
  notion of \emph{\cuc}, which is a new approximation guarantee for the
  empirical loss. While it is inspired by uniform
  convergence in statistical learning, our approach additionally incorporates a correction
  term to avoid unnecessary variance due to outliers. This
  can be viewed as a new connection between statistical learning theory and
  variance reduction techniques in stochastic optimization,
which should be of
  independent interest.
\end{abstract}

\section{Introduction}

Consider a linear regression problem defined by an $n\times d$ data matrix
$A$ and a vector $y\in\R^n$ of labels (or responses). Our goal is to
approximately find a vector $\beta\in\R^d$ that minimizes the total
loss over the entire dataset, given by
$L(\beta)=\sum_{i=1}^nl(a_i^\top\beta-y_i)$, where $a_i^\top$ is the $i$th
row of $A$. Suppose that we are only given the data matrix $A$ and we
can choose to query some number of individual labels $y_i$ (with the
remaining labels hidden). When the
cost of obtaining those labels dominates other computational costs
(for example, because it requires performing a complex and
resource-consuming measurement), 
it is natural to ask how many
queries are required to obtain a good approximation of the optimal
solution over the entire dataset. This so-called query complexity
arises in a number of statistical learning tasks such as active learning
and experimental design.
\begin{definition}[Query complexity]\label{def:query-complexity}
  For a given loss function $l:\R\rightarrow \R_{\geq 0}$, the query
  complexity $\Mc:=\Mc_l(d,\epsilon,\delta)$ is the smallest number such that
  there is a randomized algorithm that, given any $n\times d$ matrix
  $A$ and any hidden vector $y\in\R^n$, queries $\Mc$ entries of $y$
  and returns $\wt{\beta}\in\R^d$, so that:
  \begin{align*}
L(\wt{\beta})\leq (1+\epsilon)\cdot \min_\beta
    L(\beta)\quad\text{with probability $1-\delta$},\quad\text{where}\quad L(\beta) =
    \sum_{i=1}^nl(a_i^\top\beta-y_i).
  \end{align*}
\end{definition}
Note that since the algorithm has access to the entire dataset on
which it will be evaluated, it can use that information to select the
queries better than, say uniformly at random, which differentiates
this problem from the traditional sample complexity in statistical
learning theory. Also, note that we require a relative error
approximation, as opposed to an additive one. This is because we
are not restricting the range of the labels, so a relative error
approximation provides a more useful scale-invariant guarantee (this is a common
practice when approximating regression problems).   

Significant work has been dedicated to various notions of query complexity in
classification \citep[for an overview, see][]{hanneke2014theory}.
On the other hand, in the context of regression, prior 
literature has primarily focused on the special case of least squares regression,
$l(a) = a^2$, where the optimum solution has a closed form expression,
which considerably simplifies the analysis. In this setting, the query
complexity is known to be $\Mc=O(d/\epsilon)$ \citep{CP19}. 
An important drawback of the square loss in linear regression is its
sensitivity to outliers, and to that end, a number of other loss
functions are commonly used in practice to ensure robustness to
outliers. Those losses no longer yield a closed form solution, so the
techniques from least squares do not  apply. The primary example is
\emph{least absolute deviation} regression, 
also known as $\ell_1$ regression, which uses the loss function
$l(a)=|a|$. A natural way of interpolating between the robustness of
the $\ell_1$ loss and the smoothness of least squares, is to use an
$\ell_p$ loss, i.e., $l(a)=|a|^p$ for $p\in(1,2)$. Other popular
choices include the Huber loss and the Tukey loss; for more discussion on robust regression
see \cite{CW15}. The basic question that motivates this paper is:
What is the query complexity of robust regression? This question has
remained largely open even for the special case of least absolute
deviation regression.

\subsection{Main Results}

Our first main result gives nearly-matching (up to logarithmic factors) upper and lower bounds for
the query complexity of least absolute deviation regression.
While the randomized algorithm that achieves the upper bound is in
fact based on an existing importance sampling method, our key contribution is to provide a
new analysis of this method that overcomes a significant limitation of
the prior work, resulting in the first non-trivial guarantee for query complexity.
\begin{theorem}\label{t:ell1}
  The query complexity of least absolute deviation regression, $l(a)=|a|$, satisfies:
  \begin{align*}
  \Omega( (d+\log(1/\delta))/\epsilon^2)  \leq \Mc_{l}(d,\epsilon,\delta) \leq
    O( d\log(d/\epsilon\delta)/\epsilon^2).
  \end{align*}
\end{theorem}
To obtain the upper bound, we use the so-called Lewis
weight sampling originally due to \cite{Lewis1978} (see Section
\ref{sec:preli}), which can be implemented in 
time that is nearly-linear in the support size of the matrix $A$
\citep{CP15_Lewis}. Lewis weight sampling is known to provide 
a strong guarantee called the $\ell_1$ subspace embedding
property~\citep{Talagrand90,CP15_Lewis}, which 
is an important tool in randomized linear algebra for approximately solving least absolute 
deviation regression. However, this strategy only works if the
algorithm has unrestricted access to the vector 
$y$ when computing the weights (e.g., when
the motivation is computational efficiency rather than query
complexity), or if we require only a constant factor approximation
(i.e., $\epsilon >1$).%
\footnote{A constant factor
  approximation with $\eps=7$ and $\delta=0.4$ is folklore (e.g.,
  Theorem 6 in \cite{dasgupta2009sampling}). To obtain a more accurate approximation,
one can use the $\ell_1$ subspace embedding property
with respect to the matrix $[A;y]$ (which has $d+1$ columns) instead of $A$, however this requires
unrestricted access to the vector $y$.}
So, prior work on Lewis weight sampling provides
no guarantees for query complexity with $\epsilon\leq 1$.

In this paper, we develop a new analysis of Lewis weight sampling,
which relies on what we call \emph{\cuc}: an approximation guarantee for the
empirical loss that is similar to uniform
convergence in statistical learning theory, except it incorporates a correction 
to reduce the variance due to outliers. This leads to an intriguing
connection between uniform convergence in statistical learning and
variance reduction techniques in stochastic optimization (see Section \ref{s:technique}). Our new
approach not only allows us to compute the weights without accessing
the label vector, thereby obtaining guarantees for query complexity,
but it also leads to simplifications of existing methods for approximately solving 
the regression problem, both in terms of the algorithms and the
analysis (see Section \ref{s:related-work}). Our analysis framework is
not specific to the $\ell_1$ loss, and is likely of interest beyond
robust regression. To demonstrate this, our second main result
uses {\cuc} to provide an upper bound on the query complexity of regression with
any $\ell_p$ loss where $p\in(1,2)$, which is again the first
non-trivial guarantee of this kind.
\begin{theorem}\label{t:ellp}
For any $p\in(1,2)$,  the query complexity of $\ell_p$ regression,
$l(a)=|a|^p$, satisfies:
\begin{align*}
  \Mc_{l}(d,\epsilon,\delta) \leq O(d^2\log(d/\epsilon\delta)/\epsilon^2\delta).
\end{align*}
\end{theorem}
While this result also relies on Lewis weight sampling, the analysis
is considerably more challenging, since it requires interpolating between
the techniques for $\ell_1$ and $\ell_2$ losses. We expect that the upper
bound can be improved to match the one from Theorem \ref{t:ell1} in
terms of the dependence on $d$, and we leave this as a new direction for
future work.

\subsection{Key Technique: Robust Uniform Convergence}
\label{s:technique}

Next, we give a brief overview of the proof techniques used to obtain
the main results, focusing on the
notion of {\cuc} which is central to our analysis.

A key building block in the algorithms for approximately solving
regression problems is randomized sampling, which is used to construct
an estimate of the loss
$L(\beta)=\sum_{i=1}^nl(a_i^\top\beta-y_i)$. This involves
constructing a sparse sequence of non-negative random variables
$s_1,...,s_n$, such that $\wt{L}(\beta)=\sum_{i=1}^ns_i l(a_i^\top\beta-y_i)$ 
is an unbiased estimate of $L(\beta)$, i.e.,
$\E[\wt{L}(\beta)]=L(\beta)$. The algorithm then solves the regression
problem defined by the loss estimate, $\wt{\beta}=\argmin_\beta
\wt{L}(\beta)$, and returns $\wt{\beta}$ as the approximate solution
to the original problem. To minimize the number of queries, we
have to make sure that the number of non-zero $s_i$'s is small
(if $s_i>0$, then we must query $y_i$), while
at the same time preserving the quality of the approximate solution.

A standard approach for establishing guarantees in statistical
learning is via the framework of uniform convergence: Showing that
$\sup_\beta|L(\beta)-\wt{L}(\beta)|$ vanishes at some rate with the
sample size \cite[for an overview,
see][]{vapnik1999overview,vapnik2013nature}. This approach can be
adapted to the setting of relative 
error approximation, requiring that for sufficiently large sample size:
\begin{align*}
  \text{(uniform convergence of relative error)}\quad
  \sup_\beta\frac{|L(\beta)-\wt{L}(\beta)|}{L(\beta)}\leq\epsilon
  \quad\text{with probability }1-\delta.
\end{align*}
This has proven successful for importance sampling in
regression problems where the access to label vector $y$ is
unrestricted and the primary aim is computational efficiency. However, in the query model, it may happen that one
entry $y_i$ (unknown to the algorithm) is an outlier which significantly contributes to the loss
$L(\beta)$, and without sampling that entry we will not obtain a good
estimate $\wt{L}(\beta)$. In the worst-case, a randomized algorithm
that is oblivious to $y$ may need to sample almost all of the entries
before discovering the outlier. Does this mean that it is impossible to
obtain a relative error approximation of the optimum without catching
the outlier? In fact it does not, and this is particularly intuitive in
the case of least absolute deviation regression, where our goal is
specifically to ignore the outliers. For this reason, we use a
modified version of the uniform convergence property, where the
contribution of the outliers is subtracted  from the loss, using a
correction term denoted by $\Delta$ in the following definition.
\begin{definition}[\Cuc]\label{def:uniform_approx}
  A  randomized algorithm satisfies the
  {\cuc} property with query complexity
  $m(d,\epsilon,\delta)$ if, given any $n\times d$ matrix
  $A$ and hidden vector $y\in\R^n$, it queries $m(d,\epsilon,\delta)$ entries of $y$
  and then with probability $1-\delta$ returns $\wt{L}(\cdot)$ s.t.:  
  \begin{align}
    \sup_\beta\frac{|L(\beta)-\wt{L}(\beta) - \Delta|}{L(\beta)}
    \leq \epsilon,\quad\text{where}\quad \Delta = L(\beta^*) -
    \wt{L}(\beta^*),\quad \beta^*=\argmin_\beta L(\beta).\label{eq:cuc}
  \end{align}
\end{definition}
Note that as long as $\wt{L}(\cdot)$ is
an unbiased estimate of $L(\cdot)$, then $\Delta$ is a mean zero
correction of the  error quantity in uniform convergence. This bares
much similarity to variance reduction 
techniques which have gained considerable attention in stochastic
optimization \cite[for an overview, see][]{gower2020variance}. The key
difference is that we are using the correction purely for the analysis, and not for the
algorithm. Thus, whereas in optimization algorithms
such as Stochastic Variance Reduced Gradient \cite[SVRG,][]{johnson2013accelerating}, one must
explicitly compute the correction based on an estimate of $\beta^*$,
in our setting one never has to compute $\Delta$, so we can simply use
$\beta^*$ itself to define the correction. Nevertheless, the analogy is apt
in that the corrected error quantity will in fact have reduced
variance, particularly in the presence of outliers, which is what
enables our query complexity analysis.

The guarantee from \eqref{eq:cuc} immediately implies that the regression
estimate $\wt{\beta}=\argmin_\beta\wt{L}(\beta)$ is a
relative error approximation of $\beta^*$, as long as $\epsilon<1$:
\begin{align*}
  L(\wt{\beta})-L(\beta^*)\leq \wt{L}(\wt{\beta})-\wt{L}(\beta^*) +
  \epsilon\cdot L(\wt{\beta})\leq \epsilon\cdot L(\wt{\beta}),
\end{align*}
where we used that $\wt{L}(\wt{\beta})\leq
\wt{L}(\beta^*)$, and after some manipulations, we get
$L(\wt{\beta})\leq(1+\frac\epsilon{1-\epsilon})\cdot L(\beta^*)$. Thus
to bound the query complexity in Theorems \ref{t:ell1} and
\ref{t:ellp}, it suffices to bound the complexity of ensuring
{\cuc} for a given loss function. 

Finally, to obtain our results we must still use
carefully chosen
importance sampling to construct the loss estimate, where the
importance weights depend on the data matrix $A$. Here, we use the Lewis
weights, which is an extension of statistical leverage scores (see
Section \ref{sec:preli} for details) that is known to be effective in
approximating $\ell_p$ losses. Existing guarantees for Lewis
weight sampling (namely, the $\ell_p$ subspace embedding property)
prove insufficient for establishing {\cuc}, so 
we develop new techniques, which are presented for $p=1$ in
Section~\ref{sec:analysis_ell1} and for $p\in(1,2)$ in
Section~\ref{sec:analysis_ellp}. 
Moreover, in Section~\ref{sec:preli}, we discuss Lewis weights and
their connection to natural importance weights defined as
$\underset{\beta}{\max} \frac{ |a_i^{\top} \beta|^p}{\|A \beta\|_p^p}$
for each row $a_i$.  
Also, note that our results only
require a constant factor approximation of Lewis weights, where the
constant enters into the query complexity. As a consequence, we can
obtain new relative error guarantees for uniform sampling, where the sample size
depends on a notion of matrix coherence based on the degree of
non-uniformity of Lewis~weights.

\subsection{Related Work}
\label{s:related-work}

There is significant prior work related to query complexity and
relative error approximations for linear regression, which we summarize below.

Much of the work on importance sampling for linear regression has been
done in the context of Randomized Numerical Linear Algebra 
\citep[RandNLA; see, e.g.,][]{DM16_CACM,dpps-in-randnla}, where the primary goals are computational efficiency or
reducing the size of the problem. This line of work was initiated by
\cite{drineas2006sampling}, using importance sampling via statistical
leverage scores to obtain relative error approximations of $\ell_2$
regression. Leverage score sampling is known to require $\Theta(d\log d + d/\epsilon)$
queries \citep[see, e.g.,][]{correcting-bias-journal}, where we let
the failure probability $\delta$ be a small constant for
simplicity. Other techniques, such as random projections
\citep[e.g.,][]{sarlos-sketching,regression-input-sparsity-time},
have been 
used for efficiently solving regression problems, however those
methods generally require unrestricted access to the label vector
$y$. More recently, \cite{unbiased-estimates} considered the query
complexity of $\ell_2$ regression, showing that $d$ queries are
sufficient to obtain a relative error approximation with
$\epsilon=d$, by using a non-i.i.d.\ importance sampling technique
called Volume Sampling. Note that at least $d$ queries are necessary
for regression with any non-trivial loss
\citep{unbiased-estimates-journal}.  Efficient algorithms for 
Volume Sampling were given by
\cite{leveraged-volume-sampling,correcting-bias}. A different
non-i.i.d.~sampling approach was used by \cite{CP19} to reduce
the query complexity of least squares to $O(d/\epsilon)$, with a matching
lower bound.

Sampling algorithms for regression with a general $\ell_p$ loss were
studied by \cite{dasgupta2009sampling}. The importance weights they used are
closely related to Lewis weights, and most of our analysis can
be adapted to work with those weights (we use Lewis weights because
they yield a better polynomial dependence on $d$). They use a two-stage sampling scheme, where the
importance weights in the first stage are computed without accessing
the label vector. Their analysis of the first stage leads to $O(d^{2.5})$ query
complexity for obtaining a relative error approximation with
$\epsilon=O(1)$. However, to further improve the approximation, their
second stage constructs refined importance weights using the entire
label vector. Remarkably, using {\cuc} one can show that, at least for
$p\in[1,2]$, their second stage weights are not necessary (it suffices
to use the weights from the first stage with an adjusted sample size),
simplifying both the algorithm and the analysis. Other randomized
methods have been proposed for robust regression 
\citep[e.g.,
see][]{clarkson2005subgradient,iterative-row-sampling,mm-sparse,DLS18a}. In
particular, \cite{CW15} provide a more general 
framework that includes the Huber loss. However, all of these
approaches require unrestricted access to the label vector, and so
they do not imply any bounds for query complexity. Finally,
\cite{CP15_Lewis} proposed to use Lewis weights as an improvement to the
importance weights of \cite{dasgupta2009sampling}, however their results
again focus on computational efficiency.  We discuss this in
more detail in Section \ref{sec:preli}.

Query complexity has been studied extensively in the context of active
learning \citep[see,
e.g.,][]{cohn1994improving,balcan2009agnostic,hanneke2014theory}. These
works focus mostly on classification 
problems, where one can take advantage of adaptivity: selecting next
query based on the previously obtained labels. Our framework
(Definition \ref{def:query-complexity}) does allow for such adaptivity,
however, in the context of regression it appears to be of limited use 
beyond the initial selection of sampling weights. Approaches without
adaptivity are sometimes referred to as pool-based 
active learning \citep{pool-based-active-learning-regression} or
experimental design \citep{chaloner1984optimal,minimax-experimental-design}.

Finally, in classification problems where each $y_i \in \{0,1\}$,
prior work considered a variance reducing correction term to
obtain faster rates of uniform convergence \citep[for an overview,
see][]{boucheron2005theory}. Despite 
some similarities, our notion of robust uniform convergence serves a
different purpose in that it is used to
remove the effect of outliers in a regression problem where $y_i$
could be~unbounded. 

\section{Preliminaries}\label{sec:preli}
Throughout this work, let $A$ denote the data matrix of dimension $n \times d$ and $a_1,\ldots,a_n$ be its $n$ rows such that $A^{\top}= [a_1^{\top},\ldots,a_n^{\top}]$. Then let $\beta \in \mathbb{R}^d$ denote the coefficient vector and $y \in \mathbb{R}^n$ be the hidden vector of labels. When $p$ is clear, $\beta^*$ denotes the minimizer of $\ell_p$ regression $L(\beta)=\|A \beta - y\|_p^p$, where $\|\cdot\|_p$ denotes the $\ell_p$ vector norm.

We always use $\wt{L}$ to denote the empirical loss $\wt{L}(\beta)=\sum_{i=1}^n s_i \cdot |a_i^{\top} \beta - y_i|^p$ with weights $(s_1,\ldots,s_n)$. When we write down the weights as a diagonal matrix $S=\diag(s_1^{1/p},\ldots,s_n^{1/p})$, the loss becomes $\wt{L}(\beta)=\|S A \beta - S y\|_p^p$. Since it is more convenient to study the concentration of the random variable $\wt{L}(\beta)-\wt{L}(\beta^*)$ around its mean $L(\beta)-L(\beta^*)$, we will rewrite \eqref{eq:cuc} as
\[
\wt{L}(\beta)-\wt{L}(\beta^*) = L(\beta)-L(\beta^*) \pm \epsilon\cdot L(\beta)\quad\forall \beta \in \mathbb{R}^d,
\]
where $a=b \pm \epsilon$ means that $a$ and $b$ are $\eps$-close, i.e., $a \in [b-\epsilon,b+\epsilon]$.
Furthermore, for two variables (or numbers) $a$ and $b$, we use $a \approx_{\alpha} b$ to denote that $a$ is an $\alpha$-approximation of $b$ for $\alpha \ge 1$, i.e., that $b/\alpha \le a \le \alpha \cdot b$.  Also, let $\poiss(\lambda)$ denote the Poisson random variable with mean $\lambda$ and $\sign(x)$ be the sign function which is 0 for $x=0$ and $x/|x|$ for $x\neq 0$.

\paragraph{Importance weights and Lewis weights.} Given $A$ and the norm $\ell_p$, we define the importance weight of a row $a_i$ to be \[
\sup_{\beta} \frac{|a_i^{\top} \beta|^p}{\|A \beta\|_p^p}.
\]
Note that this definition is rotation free, i.e., for any rotation matrix $R \in \mathbb{R}^{d \times d}$, vector $R^\top a_i$ has the same importance weight in $AR$ since we could replace $\beta$ in the above definition by $R^{-1}\beta$. 

While most of our results hold for importance weights, it will be more convenient to work with Lewis weights \citep{Lewis1978}, which have efficient approximation algorithms  and provide additional useful properties. The Lewis weights $(w_1,\ldots,w_n)$ of $A \in \mathbb{R}^{n \times d}$ are defined as the unique solution \citep{Lewis1978,CP15_Lewis} of the following set of equations:
\begin{equation}\label{eq:def_Lewis_weight}
\forall i \in [n], a_i^{\top} \cdot (A^{\top} \cdot \diag(w_1,\ldots,w_n)^{1-2/p} \cdot A)^{-1} a_i = w_i^{2/p}.
\end{equation}
Note that Lewis weights always satisfy $\sum_i w_i=d$. We will use extensively the following relation between the importance weights and Lewis weights, which is potentially of independent interest.
\begin{theorem}\label{thm:Lewis_weight_importance}
Given any $p \in [1,2]$ and matrix $A$, let $(w_1,\ldots,w_n)$ be the Lewis weights of $A$ defined in \eqref{eq:def_Lewis_weight}. Then the importance weight of every row $i$ in $A$ is bounded by
\[
\underset{\beta \in \mathbb{R}^{d}}{\max} \frac{|a_i^{\top} \beta|^p}{\|A \beta\|_p^p} \in \left[ d^{-(1-p/2)} \cdot w_i, w_i \right].
\]
\end{theorem}
We remark that when $p=2$, the Lewis weights are equal to the statistical leverage scores of $A$ (the $i$th leverage score is defined as $a_i^\top(A^\top A)^\dagger a_i$, where $^\dagger$ denotes the Moore-Penrose pseudoinverse), which are known to be equal to the $\ell_2$  importance weights \citep{spielman2011graph,CP19}. The proof of Theorem~\ref{thm:Lewis_weight_importance} is via a study of the relationships between importance weights, statistical leverages scores, and Lewis weights. In particular, our analysis (Lemma~\ref{lem:uniform_lewis_bounds_importance} and Claim~\ref{clm:non_uniform_Lewis_w}) also gives the sample complexity of uniform sampling for $\ell_p$ regression by comparing Lewis weights with uniform sampling in terms of the non-uniformity of the statistical leverage scores. We defer the detailed discussions and proofs to Appendix~\ref{sec:Lewis_weight_importance}.

\paragraph{Computing Lewis weights.} Given two sequences of weights $(w'_1,\ldots,w'_n)$ and $(w_1,\ldots,w_n)$, we say $w'$ is an $\gamma$-approximation of $w$ if $w'_i \approx_{\gamma} w_i$ for every $i \in [n]$. Now we invoke the contraction algorithm by \cite{CP15_Lewis} to approximate $w_i$. While we state it for a $(1+\epsilon)$-approximation, our results only need a constant approximation say $\gamma=2$.
\begin{lemma}\label{lem:compute_Lewis}[Theorem 1.1 in \cite{CP15_Lewis}]
Given any matrix $A \in \mathbb{R}^{n \times d}$ and $p \in [1,2]$, there is an algorithm that runs in time $\log (\log \frac{n}{\epsilon}) \cdot O(\mathbf{nnz}(A)\log n+d^{\omega+o(1)} \big) $ to output a $(1+\epsilon)$-approximation of the Lewis weights of $A$, where $\mathbf{nnz}(A)$ denotes the number of nonzero entries in $A$ and $\omega$ is the matrix-multiplication exponent.
\end{lemma}

\paragraph{Subspace embedding.} Since the diagonal matrix $S=\diag(s_1^{1/p},\ldots,s_n^{1/p})$ will be generated from the Lewis weights with support size $O(\frac{d \log \frac{d}{\eps\delta}}{\eps^2})$ for $p=1$ and $O(\frac{d^2 \log d/\eps\delta}{\eps^2 \delta})$ for $p\in (1,2]$, by the main results in \cite{CP15_Lewis}, we can assume that $SA$ satisfies the $\ell_p$ subspace embedding property, given below.
\begin{fact}\label{fact:subspace_emd}
Given any $p \in [1,2]$ and $A$, the sketch matrix $S$ generated in Theorem~\ref{thm:concentration_contraction} for $p=1$ or Theorem~\ref{thm:property_ellp} for $p\in (1,2]$ has the following property: With probability $1-\delta$,
\[
\|S A \beta\|_p^p \approx_{1+\epsilon} \|A \beta\|_p^p \quad\text{ for any } \beta.
\]
\end{fact}

\section{Near-optimal Analysis for $\ell_1$ Regression}\label{sec:analysis_ell1}
We consider the $\ell_1$ loss function $L(\beta)=\|A \beta - y\|_1$ in this section. The main result is to prove the upper bound in Theorem~\ref{t:ell1}, namely that we can use Lewis weight sampling with support size $\wt{O}(d/\eps^2)$ to construct $\wt{L}(\beta)$ that satisfies the {\cuc} property.

\begin{theorem}\label{thm:concentration_contraction}
Given an $n\times d$ matrix $A$ and $\epsilon,\delta\in(0,1)$, let $(w'_1,\ldots,w'_n)$ be a $\gamma$-approximation of the Lewis weights $(w_1,\ldots,w_n)$ of $A$, i.e., $w'_i \approx_{\gamma} w_i$ for all $i \in [n]$. Let $u = \Theta\big( \frac{\eps^2}{\log \gamma d/\delta \eps} \big)$ and $p_i=\min\{\gamma w'_i/u, 1\}$. For each $i\in [n]$, with probability $p_i$, set $s_i=1/p_i$ (otherwise, let $s_i=0$). 

Let $L(\beta)=\|A \beta - y\|_1$ with an unknown label vector $y \in \mathbb{R}^n$. With probability $1-\delta$, the support size of $s$ is $O(\frac{\gamma^2 d \cdot \log \gamma d/\eps\delta}{\eps^2})$ and $\wt{L}(\beta):=\sum_{i=1}^n s_i \cdot |a_i^\top \beta - y_i|$ satisfies \cuc:
\[
\wt{L}(\beta)-\wt{L}(\beta^*) = L(\beta)-L(\beta^*) \pm \epsilon \cdot L(\beta) \quad\text{ for all } \beta.
\]
\end{theorem}

First of all, we bound the sample size (i.e., the support size of $s$). Since $\sum_i w_i=d$ and $w'_i \approx_{\gamma} w_i$, let $m:=\sum_i p_i$ denote the expected sample size, which is at most $\gamma/u \cdot \sum_i w'_i \le \gamma^2 d/u$. Since the variance of $|\supp(s)|$ is at most $m$, with probability $1-\delta/2$, the support size of $s$ is $O(m+ \sqrt{m \cdot \log 1/\delta} + \log 1/\delta)=O(\frac{\gamma^2 d \cdot \log \gamma d/\eps\delta}{\eps^2})$ by the Bernstein inequalities. Moreover, Lemma~\ref{lem:compute_Lewis} provides an efficient algorithm to generate $s$ with $\gamma=2$.

In the rest of this section, we prove the \cuc~property in Theorem~\ref{thm:concentration_contraction}. Before showing the formal proof in Section~\ref{sec:concentration_contraction}, we discuss the main technical tool --- a concentration bound for all $\beta$ with bounded $\|A (\beta-\beta^*)\|_1$. 

\begin{lemma}\label{lem:high_prob_guarantee}
Given $A \in \mathbb{R}^{n \times d}$, $\epsilon$ and a failure probability $\delta$, let $w_1,\ldots,w_n$ be the Lewis weight of each row $A_i$. Let $p_1,\ldots,p_n$ be a sequence of numbers upper bounding $w$, i.e.,
\[
p_i \ge \min\{ w_i/u, 1\} \text{ for } u =\Theta\left( \frac{\eps^2}{\log (m/\delta+d/\eps\delta)} \right) \text{ where } m = \sum_i p_i.
\]
Suppose that the coefficients $(s_1,\ldots,s_n)$ in $\wt{L}(\beta):=\sum_{i=1}^n s_i \cdot |a_i^{\top} \beta - y_i|$ are generated as follows: For each $i \in [n]$, with probability $p_i$, we set $s_i=1/p_i$. Given any subset $B \subset \mathbb{R}^d$ of $\beta$, with probability at least $1-\delta$,
\[
\wt{L}(\beta)-\wt{L}(\beta^*) = L(\beta)-L(\beta^*) \pm \eps \cdot \sup_{\beta \in B}\|A (\beta-\beta^*)\|_1 \quad \text{ for all } \beta \in B.
\]
\end{lemma}

We remark that Lemma~\ref{lem:high_prob_guarantee} holds for any choice of $B$ while the error depends on its radius via the term $\eps \cdot \sup_{\beta \in B}\|A (\beta-\beta^*)\|_1$. For example, we could apply this lemma to $B:=\big\{ \beta \mid \|A (\beta-\beta^*)\|_1 \le 3 L(\beta^*) \big\}$ to conclude the error is at most $\epsilon \cdot \|A (\beta-\beta^*)\|_1 \le 3 \epsilon\cdot L(\beta^*)$, which is at most $3 \epsilon \cdot L(\beta)$ by the definition of $\beta^*$. However, this only gives guarantee for bounded $\beta$. To bound the error in terms of $\eps L(\beta)$ for \emph{all} $\beta \in \mathbb{R}^d$ with high probability, we partition $\beta \in \mathbb{R}^d$ into several subsets and apply Lemma~\ref{lem:high_prob_guarantee} to those subsets separately. We defer the detailed proof of Theorem~\ref{thm:concentration_contraction} to Section~\ref{sec:concentration_contraction}. 

The proof of Lemma~\ref{lem:high_prob_guarantee} is based on the contraction principle from~\cite{LTbook} and chaining arguments introduced in \cite{Talagrand90, CP15_Lewis}. The main observation in its proof is that after rewriting $\wt{L}(\beta)-\wt{L}(\beta^*)=\sum_{i=1}^n s_i \cdot (|a_i^{\top} \beta - y_i|-|a_i^{\top} \beta^* - y_i|)$, all vectors in this summation, $\big( |a_i^{\top} \beta - y_i|-|a_i^{\top} \beta^* - y_i| \big)_{i \in [n]}$, contract from the vector $\big( |a_i^{\top} (\beta-\beta^*) |\big)_{i \in [n]}$ in the $\ell_1$ summation. So, using the contraction principle \citep{LTbook}, $\wt{L}(\beta)-\wt{L}(\beta^*)$ will concentrate at least as well as $\|SA(\beta-\beta^*)\|_1$, which can be analyzed using the $\ell_1$ subspace embedding property (see Fact \ref{fact:subspace_emd}). We defer the formal proof to Appendix~\ref{sec:proof_thm_deviation}. 

\subsection{Proof of Theorem~\ref{thm:concentration_contraction}}\label{sec:concentration_contraction}
Recall that $m=\sum_i p_i \le \gamma^2 d / u$ and $u =\Theta\left( \frac{\eps^2}{d/\eps\delta} \right)$. Since $\log m/\eps\delta=O(\log d\gamma /\eps\delta)$, we choose a small constant for $u$ such that for some large $C$, $
u \le \frac{\eps^2}{C \log (m/\eps\delta +d/\eps\delta)} \text{ and } p_i \ge \min\{1,w_i/u\}$.

For convenience, and without loss of generality, we assume $\beta^*=0$ such that~$L(\beta^*)=\|y\|_1$ (this is because we can always shift $y$ to $y- A \beta^*$ and $\beta^*$ to $0$). Thus we can write \[
\wt{L}(\beta)-\wt{L}(\beta^*)=\|S A \beta - Sy\|_1 - \|S y\|_1 \quad\text{ and }\quad L(\beta)-L(\beta^*)=\|A \beta - y\|_1 - \|y\|_1
\] so that we can use the triangle inequality of the $\ell_1$ norm. Next, notice that $
\E_S[\|S y\|_1]=\|y\|_1.$
By Markov's inequality, $\|S y\|_1 \le \frac{1}{\delta} \cdot \|y\|_1$ holds with probability $1-\delta$. 
We now apply Lemma~\ref{lem:high_prob_guarantee} multiple times with different choices of $B_{\beta}$. We choose $t=O(1/\eps^2\delta)$ and $B_i=\big\{ \beta \mid \|A \beta\|_1 \le (3+\eps \cdot t) \|y\|_1 \big\}$ for $i=0,1,\ldots,t$. Then we apply Lemma~\ref{lem:high_prob_guarantee} with failure probability $\frac{\delta}{t+1}=O(\delta^2 \epsilon^2)$ to guarantee that with probability $1-\delta$, we have for \emph{all} $i=0,1,\ldots,t$: 
\[
  \|S A \beta - Sy\|_1 - \|S y\|_1 = \|A \beta\|_1 - \|y\|_1 \pm \eps \cdot \sup_{\beta \in B_i} \|A \beta\|_1 \quad\text{ for all } \beta \in B_i.
\]
Furthermore, with probability $1-\delta$ over $S$, we have the $\ell_1$ subspace embedding property:
\[
\forall \beta, \|S A \beta\|_1 = (1\pm \eps) \cdot \|A \beta\|_1.
\]
We can now argue that the concentration holds for all $\beta$, by considering the following three cases:
\begin{enumerate}
\item $\beta$ has $\|A \beta\|_1 < 3 \|y\|_1$: From the concentration of $B_0$, the error is $\leq\eps \cdot 3 \|y\|_1 \le 3\eps \cdot L(\beta)$.

\item $\beta$ has $\|A \beta\|_1 \in \big[(3+\eps \cdot i) \cdot \|y\|_1, (3+ \eps \cdot (i+1))\cdot \|y\|_1 \big)$ for $i<t$: Let $\beta'$ be the rescaling of $\beta$ with $\|A \beta'\|_1=(3+\eps \cdot i) \|y\|_1$ such that $\|A (\beta -\beta')\|_1 \le \eps \|y\|_1$. Then we rewrite $\|A \beta - y\|_1 - \|y\|_1$ as
\[
\|A (\beta - \beta' + \beta') - y\|_1 - \|y\|_1 = \|A \beta' - y\|_1 - \|y\|_1 \pm \|A (\beta-\beta')\|_1 = \|A \beta' - y\|_1 - \|y\|_1 \pm \eps \|y\|_1.
\]
Similarly, we rewrite $\wt{L}(\beta)-\wt{L}(\beta^*)$ as $\|S A \beta - S y\|_1 - \|S y\|_1$ and bound it by
\begin{align*}
 \|S A \beta' - S y\|_1 - \|S y\|_1 \pm \|S A (\beta-\beta')\|_1 & = \|S A \beta' - S y\|_1 - \|S y\|_1 \pm (1+\eps ) \cdot \|A(\beta-\beta')\|_1 \\ 
 & =\|S A \beta' - S y\|_1 - \|S y\|_1 \pm (1+\eps ) \cdot \eps\|y\|_1,
\end{align*}
where we use the $\ell_1$ subspace embedding in the middle step. Next we use the guarantee of $\beta' \in B_i$ to bound the error between $\|A \beta' - y\|_1 - \|y\|_1$ and $\|S A \beta' - S y\|_1$ by $\eps \cdot \|A \beta'\|_1$. So the total error is $O(\eps) \cdot (\|A \beta'\|_1 +\|y\|_1) = O(\eps) \cdot L(\beta)$ since $L(\beta) \ge \|A \beta'\|_1/2$ by the triangle inequality.

\item $\beta$ has $\|A \beta\|_1 \ge \frac{25}{\eps \delta} \cdot \|y\|_1$: We always have $\|A \beta - y\|_1 - \|y\|_1 = \|A \beta\|_1 \pm 2 \|y\|_1$ by the triangle inequality. On the other hand, the $\ell_1$ subspace embedding and the bound on $\|Sy\|_1$ imply: 
\[
\|SA\beta - S y\|_1 - \|S y\|_1 = \| S A \beta\|_1 \pm 2 \|S y\|_1= (1 \pm \eps)\|A \beta\|_1 \pm \frac{2}{\delta} \|y\|_1.
\] Since $\|A \beta\|_1 \ge \frac{25}{\eps \delta} \cdot \|y\|_1$, this becomes $(1 \pm \epsilon) \|A\beta\|_1 \pm \frac{2}{\delta} \cdot \frac{\eps \delta}{25} \|A \beta\|_1 = (1 \pm \epsilon \pm \frac{2\epsilon}{25}) \|A \beta\|_1$.
Again, the error is $2\eps \|A \beta\|_1 = O(\eps) \cdot L(\beta)/2$ by the triangle inequality.
\end{enumerate}

\section{General Analysis for $\ell_p$ Regression}\label{sec:analysis_ellp}
In this section, we consider the $\ell_p$ loss function $L(\beta):=\sum_{i=1}^n |a_i^{\top} \beta - y_i|^p$ for a given $p \in (1,2]$. The main result is to show the upper bound in Theorem~\ref{t:ellp}, i.e., that $\wt{L}(\beta)$, when generated properly according to the Lewis weights with sample size $\wt{O}(\frac{d^2}{\eps^2 \delta})$, satisfies {\cuc}. 

\begin{theorem}\label{thm:property_ellp}
Given $p \in (1,2]$ and a matrix $A \in \mathbb{R}^{n \times d}$, let $(w'_1,\ldots,w'_n)$ be a $\gamma$-approximation of the Lewis weights $(w_1,\ldots,w_n)$ of $A$, i.e., $w_i \approx_{\gamma} w'_i$ for all $i \in [n]$. For $m=O \big( \frac{\gamma \cdot d^2 \log d/\eps\delta}{\eps^2} + \frac{ \gamma \cdot d^{2/p}}{\eps^2 \delta} \big)$ with a sufficiently large constant, we sample $s_i \sim \frac{d}{m \cdot w'_i} \cdot \poiss(\frac{m \cdot w'_i}{d})$ for each $i \in [n]$.

Then with probability $1-\delta$, $|\supp(s)|=O(m)$ and $\wt{L}(\beta):=\sum_{i=1}^n s_i \cdot |a_i^{\top} \beta - y_i|^p$ satisfies \cuc. Namely,
\[ \wt{L}(\beta) - \wt{L}(\beta^*) = L(\beta)-L(\beta^*) \pm \eps \cdot L(\beta) \quad\text{ for any $\beta \in \mathbb{R}^d$}. \]
\end{theorem}

First of all, we bound the support size of $s$. Note that $\Pr[s_i>0]= 1- e^{-\frac{m \cdot w_i'}{d}} \le \frac{m \cdot w_i'}{d}$. So $\E[|\supp(s)|] \le \sum_i \frac{m \cdot w'_i}{d} \le \gamma \cdot m$. Since the variance of $|\supp(s)|$ is bounded by $\gamma \cdot m$, with probability $1-\delta/2$, the sample size is $O(\gamma m+ \sqrt{\gamma m \cdot \log 1/\delta} + \log 1/\delta)=O(\gamma m)$ by the Bernstein inequalities.
Also, Lemma~\ref{lem:compute_Lewis} provides an efficient algorithm to generate $s$ with $\gamma=2$.

In the rest of this section, we outline the proof of the \cuc~property in Theorem~\ref{thm:property_ellp}, while the formal proof is deferred to Appendix~\ref{sec:property_uniform_Lewis}. The proof has two steps. 

The 1st step is a reduction to the case of uniform Lewis weights.
Specifically, we create an equivalent problem with matrix $A' \in \mathbb{R}^{N \times d}$ and $y' \in \mathbb{R}^{N}$ such that $L(\beta) = \|A' \beta - y'\|^p_p$ and the Lewis weights of $A'$ are almost uniform. Moreover, we show an equivalent way to generate $(s_1,\ldots,s_n)$ so that  $\wt{L}(\beta)=\|S'A'\beta - S' y'\|_p^p$. Thus, showing that the new empirical loss $\|S' A' \beta -S'  y'\|_p^p$ satisfies {\cuc} with respect to $\|A' \beta - y'\|_p^p$ will imply the same property for $SA$. 

In the 2nd step, we interpolate techniques for $\ell_1$ and $\ell_2$ losses to prove the property for the special case of almost uniform Lewis weights (namely $A'$ in the above paragraph). Let us discuss the key ideas in this interpolation. For ease of exposition, we assume $\beta^*=0$ and $A^{\top} A = I$. Furthermore let $\alpha$ denote the approximation parameter of the uniform sampling compared to the Lewis weights of $A$ after the 1st step reduction, i.e., $w_i \approx_{\alpha} d/n$. 

We will again consider the \cuc~as a concentration of $\wt{L}(\beta)-\wt{L}(\beta^*)$, which can be written as $\sum_i s_i \cdot \big( |a_i^\top \beta - y_i|^p - |y_i|^p \big)$ from the assumption $\beta^*=0$. Our proof heavily relies on the following two properties of $A$ when its Lewis weights are almost uniform.
\begin{fact}\label{fact:almost_uniform}
   Let $A$ be an $n\times d$ matrix such that $A^\top A=I$. If $w_i \approx_{\alpha} d/n$ for every $i\in [n]$, then:
   \begin{enumerate}
   \item
     The leverage scores are almost uniform (recall that when $A^{\top} A=I$, the $i$th leverage score becomes $\|a_i\|_2^2$):
\begin{equation}\label{eq:leverage_score_ell2}
\|a_i\|_2^2 \approx_{\alpha^{C_p}} d/n \quad \text{ for } \quad C_p=4/p-1.
\end{equation}
\item The importance weights of $\|A \beta\|^p_p$ are almost uniform:
\begin{equation}\label{eq:importance_ellp}
\forall \beta, \quad \forall i \in [n], \quad \frac{\|a_i^{\top} \beta\|_p^p}{\|A \beta\|_p^p} \le w_i \le \alpha \cdot d/n.
\end{equation}
\end{enumerate}
\end{fact}
The first property is implied by Claim~\ref{clm:non_uniform_Lewis_w} in Appendix~\ref{sec:Lewis_weight_importance} and the 2nd one is from Theorem~\ref{thm:Lewis_weight_importance}.

The major difference between the analyses of $p>1$ and $p=1$ is that we cannot use the contraction principle to remove the effect of outliers in $\sum_{i} s_i \cdot |a_i^{\top} \beta - y_i|^p$. For example, when $p=2$, we always have the cross term $2\sum_i s_i \cdot (a_i^{\top} \beta) y_i$ besides $\sum_i s_i \big((a_i^{\top} \beta)^2 + y_i^2 \big)$. For general $p\in(1,2)$, the cross term is replaced by the 1st order Taylor expansion of $\sum_{i} s_i \cdot |a_i^{\top} \beta - y_i|^p$. While the assumption that $\beta^*=0$ implies that its expectation is $\sum_i (a_i^{\top} \beta) y_i=0$, we can only use Chebyshev's inequality to conclude that the cross term is small for \emph{a fixed} $\beta$. Since the coefficient $s_i$ is independent of $y_i$, we cannot rely on a stronger Chernoff/Bernstein type concentration. This turns out to be insufficient to use the union bound \emph{for all} $\beta$. 

So our key technical contribution in the 2nd step is to bound the cross term for all $\beta$ simultaneously. To do this, we rewrite it as an inner product between $\beta$ and $( \langle A[*,j], S^p \cdot y \rangle )_{j \in [d]}$ (say $p=2$), where $A[*,j]$ denotes the $j$th column of $A$. To save the union bound for $\beta$, we bound the $\ell_2$ norms of these two vectors separately to apply the Cauchy-Schwartz inequality. While prior work \citep{CP19} for $\ell_2$ loss provides a way to bound the 2nd vector $( \langle A[*,j], S^p \cdot y \rangle )_{j \in [d]}$ given bounded leverage scores in \eqref{eq:leverage_score_ell2}, the new ingredient is to bound $\|\beta\|_2$ in terms of $\|A \beta\|_p$ for $p<2$ when the Lewis weights are uniform. This is summarized in the following claim, stated for general $p \in (1,2)$, in which we bound the 1st order Taylor expansion of $\sum_{i} s_i \cdot |a_i^{\top} \beta - y_i|^p$.
\begin{claim}\label{clm:inner_product}
  If $w_i \approx_{\alpha} d/n$ for every $i\in [n]$, then with probability at least $1-\delta$, we have the following bound for all $\beta\in\R^d$:
  \vspace{-3mm}
\begin{equation}\label{eq:bound_inner}
\sum_{i=1}^n s_i \cdot p \cdot |y_i|^{p-1} \cdot \sign(y_i) \cdot a_i^{\top} \beta = O\bigg(\sqrt{\frac{\alpha^{\frac{2+p}{p}} \cdot \gamma \cdot d^{2/p}}{\delta \cdot m}} \cdot \|A \beta\|_p \cdot \|y\|_p^{p-1}\bigg).
\end{equation}
\end{claim}
The rest of the proof (given in Appendix~\ref{sec:property_uniform_Lewis}) is centered on obtaining a strong concentration bound for the 2nd order Taylor expansion using \eqref{eq:importance_ellp}. We remark that in the special case of $p=2$, this analysis simplifies considerably, and it can be easily adapted to show an $O(\frac{d\log d}{\delta\epsilon^2})$ bound on the query complexity of {\cuc} for least squares regression.



\section{Lower Bound for $\ell_1$ Regression}\label{sec:lower_bound}
In this section, we prove the information-theoretic lower bound in Theorem~\ref{t:ell1} for the query complexity of $\ell_1$ regression.
\begin{theorem}\label{thm:information_lower_bound}
  Given any $\epsilon$, $\delta<0.01$, and $d$, consider the $n \times d$ matrix $A$ defined as follows: \[A^{\top}=[\underbrace{e_1^{\top},\ldots,e_1^{\top}}_{n/d},e_2,\ldots,e_{d-1},\underbrace{e_d^{\top},\ldots,e_d^{\top}}_{n/d}]\]
  for the canonical basis $e_1,\ldots,e_d \in \mathbb{R}^d$ with a sufficiently large $n$. It takes $m=\Omega(\frac{d + \log 1/\delta}{\eps^2})$ queries to $y$ to produce $\tilde{\beta}$ satisfying $\|A \tilde{\beta}- y\|_1 \le (1+\epsilon) \cdot \|A \beta^* - y\|_1$ with probability $1-\delta$.
\end{theorem}
Our proof will make a reduction to the classical biased coin testing problem \citep{Kleinberg}, by relying on Yao's minmax principle so that we can consider a deterministic querying algorithm and a randomized label vector $y$.

We rely on the fact that the $\ell_1$ regression problem $\|A \beta - y\|_1$ for $A$ specified in Theorem~\ref{thm:information_lower_bound} can be separated into $d$ independent subproblems of dimension $n' \times 1$ for $n':=n/d$. First of all, let us consider the 1-dimensional subproblem, which is a reformulation of the biased coin test. In this subproblem, we will use the following two distributions with different biases to generate $y' \in \{\pm 1\}^{n'}$:
\begin{enumerate}
\item $D_1:$ Each label $y'_i=1$ with probability $1/2+\epsilon$ and $-1$ with probability $1/2-\epsilon$.
\item $D_{-1}:$ Each label $y'_i=1$ with probability $1/2-\epsilon$ and $-1$ with probability $1/2+\epsilon$.
\end{enumerate}
The starting point of our reduction is that a $(1+\eps)$-approximate solution of $\beta^*$ could distinguish between the two biased coins of $D_1$ and $D_{-1}$.
\begin{claim}\label{clm:distinguish_two_dist}
  Let $d=1$, $n' \ge \frac{100 \log 1/\delta}{\epsilon^2}$, and $A'=[1,\ldots,1]^{\top}$ whose dimension is $n' \times 1$. Let $L(\beta)$ for $\beta \in \mathbb{R}$ denote $\|A' \beta - y'\|_1$ for $y' \in \mathbb{R}^{n'}$ and $\beta^*$ be the minimizer of $L(\beta)$.

  If $y'$ is generated from $D_1$, then with probability $1-\frac{\delta}{100}$, we will have $\beta^*=1$ and any $(1+\eps)$-approximation $\tilde{\beta}$ will be positive. Similarly, if $y' \sim D_{-1}$, then with probability $1-\frac{\delta}{100}$  we will have $\beta^*=-1$ and any $(1+\eps)$-approximation $\tilde{\beta}$ will be negative.
\end{claim}
\begin{proof}
Let us consider $y'$ generated from $D_1$ of dimension $n' \ge \frac{100 \log 1/\delta}{\epsilon^2}$. Let $n_1$ and $n_{-1}$ denote the number of $1$s and $-1$s in $y'$. From the standard concentration bound, with probability $1-\frac{\delta}{100}$, we have $n_1 \ge (1/2+\epsilon/2)n'$ for a random string $y'$ from $D_1$. So for this 1-dimensional problem, $\beta^*$ minimizing $\|A' \beta-y'\|_1$ will equal $1$ with loss $L(\beta^*):=\|A' \beta^* - y'\|_1 = 2 \cdot n_{-1} \le (1-\epsilon)n$. Let $\tilde{\beta}$ be any $(1+\eps)$-approximation s.t. $\|A' \tilde{\beta}- y'\|_1 \le (1+\epsilon) \cdot \|A' \beta^* - y\|_1$. 

We show that $\tilde{\beta}$ must be positive given $n_1 \ge (1/2+\epsilon/2)n$. If $\tilde{\beta} \in (-1,0]$, then the loss of $\tilde{\beta}$ is given by
\[
n_1 \cdot (1-\tilde{\beta}) + n_{-1} \cdot (1+\tilde{\beta}) = n - \tilde{\beta} (n_1 - n_{-1}) \ge n\geq \frac{L(\beta^*)}{1-\epsilon}.
\]
Otherwise if $\tilde{\beta} \le -1$, the $\ell_1$ loss of $\tilde{\beta}$ becomes
\[
n_1 \cdot (1-\tilde{\beta}) + n_{-1} \cdot (-\tilde{\beta} - 1)=-\tilde{\beta} \cdot n + n_1 - n_{-1} \ge n\geq \frac{L(\beta^*)}{1-\epsilon}.
\]
In both cases, $L(\tilde{\beta})$ does not give a $(1+\epsilon)$-approximation of $L(\beta^*)$. Similarly when $y$ is generated from $D_{-1}$ and $n_{-1} \ge (1/2+\epsilon/2)n$, then any $\tilde{\beta}$ that is a $(1+\eps)$-approximation of $\beta^*$, must be negative.
\end{proof}

Given the connection between the 1-dimensional subproblem and the biased coin test in Claim~\ref{clm:distinguish_two_dist}, the $\Omega(\frac{\log 1/\delta}{\eps^2})$ lower bound in the main theorem comes from the classical lower bound for distinguishing between biased coins.
\begin{lemma}\label{lem:information_lower_dim_1}
Let us consider the following game between Alice and Bob:
\begin{enumerate}
\item Let Alice generate $\alpha \sim \{\pm 1\}$ randomly.
\item Alice generates $y \in \{\pm 1\}^{n'}$ from $D_\alpha$ defined above and allows Bob to query $m'$ entries in $y$.
\item After all queries, Bob wins the game only if he predicts $\alpha$ correctly.
\end{enumerate}
If Bob wants to win this game with probability $1-\delta$, he needs to make $m'=\Omega(\frac{\log 1/\delta}{\eps^2})$ queries.
\end{lemma}
The proof of Lemma~\ref{lem:information_lower_dim_1} follows from KL divergence or the variation distance (a.k.a. statistical distance). We refer to Theorem 1.5 in \cite{Kleinberg} for a proof with minor modifications such as replacing the fair coin by one of bias $1/2 - \eps$.

Next, we discuss how to obtain the $\Omega(d/\eps^2)$ part. Due to the space constraint, we give a high-level overview and defer the formal proof of Theorem~\ref{thm:information_lower_bound} to Appendix~\ref{app:proof_lower}. First of all, for each string in $b \in \{\pm 1\}^d$, we consider the $n$-dimensional distribution of $y$ generated as $D_b=D_{b_1} \circ D_{b_2} \circ \cdots \circ D_{b_d}$. Namely, each chunk of $n/d$ bits of $y$ is generated from $D_1$ or $D_{-1}$ mentioned above. 

The second observation is that when we generate $y \sim D_b$ for a random $b \in \{\pm 1\}^n$, then a $(1+\eps)$-approximation algorithm could decode almost all entries of $b$ except a tiny fraction (w.h.p. based on Claim~\ref{clm:distinguish_two_dist}). Then we show the existence of $b \in \{\pm 1\}^d$ and its flip $b^{(i)}$ on the $i$th coordinate such that: (1) the algorithm could decode entry $i$ in $b$ and $b^{(i)}$ (w.h.p. separately) when $y \sim D_b$ or $D_{b^{(i)}}$; (2) the algorithm makes $O(m/d)$ queries to decode entry $b_i$. Such a pair of $b$ and $b^{(i)}$ provides a good strategy for Bob to win the game in Lemma~\ref{lem:information_lower_dim_1} with a constant high probability using $O(m/d)$ queries, which gives a lower bound of $m$.

\section{Conclusions and Future Directions}
We provided nearly-matching upper and lower bounds for the query
complexity of least absolute deviation regression. In this setting, the learner is
allowed to use a sampling distribution that depends on the unlabeled
data. In the process, we
proposed \emph{\cuc}, a framework for showing sample complexity
guarantees in regression tasks, which includes a correction term that minimizes the effect
of outliers. Further, we extended our results to robust regression with any
$\ell_p$ loss for $p\in(1,2)$. Note that the guarantees in this paper
also apply to data-oblivious sampling, in which case we pay an additional
data-dependent factor in the sample complexity.

Many new directions for future work arise from our results. First of
all, we expect that the query complexity of $\ell_p$ regression
should match the guarantee for $\ell_1$ (our current bounds exhibit
quadratic dependence on the input dimension $d$ for $p>1$, and linear dependence
for $p=1$). Moreover, the query complexity question remains open for a
number of other robust losses, such as the Huber loss and the
Tukey loss. More broadly, we ask whether the {\cuc} framework can be
used to obtain sample complexity bounds that are robust to the presence of
outliers for other settings in
statistical 
learning, e.g., for classification and non-linear regression problems,
and with additive as well as multiplicative error bounds.

\subsection*{Acknowledgments}
We would like to acknowledge DARPA, IARPA, NSF, and ONR via its BRC on
RandNLA for providing partial support of this work.  Our conclusions
do not necessarily reflect the position or the policy of our sponsors,
and no official endorsement should be inferred.


\bibliographystyle{apalike} 
\bibliography{all.bib,../pap.bib}

\appendix

\section{Lewis Weights Bound Importance Weights}\label{sec:Lewis_weight_importance}
We finish the proof of Theorem~\ref{thm:Lewis_weight_importance} in this section. Specifically, given any matrix $A$, we show that Lewis weights $(w_1,\ldots,w_n)$ defined in Equation~\eqref{eq:def_Lewis_weight} bound the importance weights for each row:
\[
\underset{x \in \mathbb{R}^{d}}{\max} \frac{|a_i^{\top} x|^p}{\|A x\|_p^p} \in \left[ d^{-(1-p/2)} \cdot w_i, w_i \right].
\]

\begin{remark}
There are matrices that obtain both upper and lower bounds in the above inequality (up to constants). For example, $A=I$ gives the upper bound.

For the lower bound, consider $A \in \{\pm 1\}^{2^d \times d}$ constituted by $n=2^d$ distinct vectors in $\{\pm 1/\sqrt{d}\}^d$. Then for each $a_i$, we set $x=a_i$ such that $a_i^{\top} x=1$ but $w_i=d/n$ and $\|Ax\|_p^p \approx n \cdot \underset{Z \sim N(0,1/d)}{\E}[|Z|^p] = \Theta(n \cdot d^{-p/2})$.
\end{remark}

To prove Theorem~\ref{thm:Lewis_weight_importance}, we will show the following lemma for the special case where the Lewis weights are almost uniform and make a reduction from the general case to the almost uniform case. Recall that $a \approx_{\alpha} b$ means $a \in [b/\alpha,b \cdot \alpha]$ and $W=\diag(w_1,\ldots,w_n)$ denotes the diagonal Lewis-weight matrix when the matrix $A$ and parameter $p$ are fixed.
\begin{lemma}\label{lem:uniform_lewis_bounds_importance}
Suppose the $\ell_p$ Lewis weights of $A$ are almost uniform: for a parameter $\alpha \ge 1$, $w_i \approx_{\alpha} d/n$. Then for every row $i$, its importance weight satisfies:
\[
\underset{x \in \mathbb{R}^{d}}{\max} \frac{|a_i^{\top} x|^p}{\|A x\|^p_p} \in \left[ \alpha^{-O(1)} \cdot d^{p/2}/n, \alpha^{O(1)} \cdot d/n \right].
\]
\end{lemma}
In particular, if $\alpha=1$ (the Lewis weights are uniform), this indicates the importance weights are uniformly upper bounded. We defer the proof to Appendix~\ref{sec:almost_uniform_Lewis}. We remark that Claim~\ref{clm:non_uniform_Lewis_w} in Appendix~\ref{sec:almost_uniform_Lewis} shows that when the statistical leverage scores are almost uniform --- $a_i^{\top} (A^{\top} A)^{-1} a_i \approx_{\alpha} d/n$ for some $\alpha$, the $\ell_p$ Lewis weights have $w_i \approx_{\alpha^{O_p(1)}} d/n$ for all $i \in [n]$ and any $p \in [1,2]$. This bounds the sample complexity of uniform sampling in terms of $\alpha$ by plugging $\gamma=\alpha^{O(1)}$ in Theorem~\ref{thm:concentration_contraction} and Theorem~\ref{thm:property_ellp}.

While this could give a bound on the importance weights in terms of the Lewis weights for the non-uniform case, we will use the following properties to obtain the tight bound in Theorem~\ref{thm:Lewis_weight_importance} via a reduction.
\begin{claim}\label{clm:split_Lewis_weight}
Given $A \in \mathbb{R}^{n \times d}$ whose Lewis weights are $(w_1,\ldots,w_n)$, let $A' \in \mathbb{R}^{(n+k-1) \times d}$ be the matrix of splitting one row, say the last row $a_n$, into $k$ copies: $a'_i=a_i$ for $i<n$ and $a'_i=a_n/k^{1/p}$ for $i \ge n$. Then the Lewis weights of $A'$ are $(w_1,\ldots,w_{n-1},w_n/k,\ldots,w_n/k)$.
\end{claim}

And the same property holds for the importance weight.
\begin{claim}\label{clm:split_importance}
Given $A \in \mathbb{R}^{n \times d}$ whose importance weights are $(u_1,\ldots,u_n)$, let $A' \in \mathbb{R}^{(n+k-1) \times d}$ be the matrix of splitting the last row $a_n$ into $k$ copies: $a'_i=a_i$ for $i<n$ and $a'_i=a_n/k^{1/p}$ for $i \ge n$. Then the importance weights of $A'$ are $(u_1,\ldots,u_{n-1},u_n/k,\ldots,u_n/k)$.
\end{claim}

We defer the proofs of these two claims to Section~\ref{sec:proofs_claims_splitting} and finish the proof of Theorem~\ref{thm:Lewis_weight_importance}.

\begin{proofof}{Theorem~\ref{thm:Lewis_weight_importance}}
Suppose the Lewis weights of $A$ are $(w_1,\ldots,w_n)$ and the importance weights are $(u_1,\ldots,u_n)$. Let $\eps$ be a tiny constant and $N_1=\lceil w_1/\eps \rceil,\ldots,N_n=\lceil w_n/\eps \rceil$. We define $A' \in \mathbb{R}^{N \times d}$ with $N=\sum_i N_i$ as
\[
\begin{bmatrix} 
a_1/N_1^{1/p} \\
\vdots \\
a_1/N_1^{1/p} \\
a_2/N_2^{1/p} \\
\vdots \\
a_n/N_n^{1/p}
\end{bmatrix}
\]
where there are $N_1$ rows of $a_1/N_1^{1/p}$, $N_2$ rows of $a_2/N_2^{1/p}$, and so on. From Claim~\ref{clm:split_Lewis_weight}, we have the Lewis weights of $A'$ are
\[
\left( \underbrace{w_1/N_1,\ldots,w_1/N_1}_{N_1},\underbrace{w_2/N_2,\ldots,w_2/N_2}_{N_2},\ldots,\underbrace{w_n/N_n,\ldots,w_n/N_n}_{N_n} \right).
\]
Since $w_1,\ldots,w_n$ are fixed, we know $w_i/N_i=w_i/\lceil w_i/\eps \rceil \in [\eps/(1+\eps/w_1),\eps]$. So let $\alpha$ be the parameter satisfying $w_i/N_i \approx_{\alpha} d/N$. If $w_1,\ldots,w_n$ are multiples of $\eps$, $d/N=\epsilon$ and $\alpha=1$. Since $\lim_{\eps \rightarrow 0} d/N=\eps$ (by the property $\sum_i w_i=d$ and the definition of $N$), we know $\lim_{\eps \rightarrow 0} \alpha=1$.

From Claim~\ref{clm:split_importance}, we have the importance weights of $A'$ are
\[
\left( \underbrace{u_1/N_1,\ldots,u_1/N_1}_{N_1},\underbrace{u_2/N_2,\ldots,u_2/N_2}_{N_2},\ldots,\underbrace{u_n/N_n,\ldots,u_n/N_n}_{N_n} \right).
\]

By Lemma~\ref{lem:uniform_lewis_bounds_importance}, we have $w_i/N_1 \in [\alpha^{-O(1)} \cdot d^{-(1-p/2)} \cdot u_1/N_1  , \alpha^{O(1)} \cdot u_1/N_1]$ for each $i$. By taking $\eps \rightarrow 0$ and $\alpha \rightarrow 1$, this shows both upper and lower bounds.
\end{proofof}

\subsection{Proofs of Claim~\ref{clm:split_Lewis_weight} and~\ref{clm:split_importance}}\label{sec:proofs_claims_splitting}
To prove Claim~\ref{clm:split_Lewis_weight}, we use the property that Lewis weights constitute the unique diagonal matrix satisfying \eqref{eq:def_Lewis_weight} \citep{CP15_Lewis}. So we only need to verify $(w_1,\ldots,w_{n-1},w_n/k,\ldots,w_n/k)$ satisfying \eqref{eq:def_Lewis_weight} for $A'$. First of all, we show the inverse in the equation is the same:
\begin{align*}
A'^{\top} (W')^{1-2/p} A' &=\sum_{i=1}^{n+k-1} (w'_i)^{1-2/p} \cdot (a'_i) \cdot (a'_i)^{\top} \\
& =\sum_{i=1}^{n-1} (w_i)^{1-2/p} a_i \cdot a_i^{\top} + \sum_{i=1}^k (w_n/k)^{1-2/p} \cdot (a_n/k^{1/p})\cdot (a_n/k^{1/p})^{\top} \\
& = \sum_{i=1}^{n-1} (w_i)^{1-2/p} a_i \cdot a_i^{\top} + k^{2/p} \cdot w_n^{1-2/p} \cdot (a_n/k^{1/p})\cdot (a_n/k^{1/p})^{\top} = \sum_{i=1}^{n} (w_i)^{1-2/p} a_i \cdot a_i^{\top}.
\end{align*}
Then it is straightforward to verify $(a_n/k^{1/p})^{\top} (A'^{\top} W'^{-1} A')^{1-2/p} \cdot (a_n/k^{1/p})=(w_n/k)^{2/p}$.

Next We prove Claim~\ref{clm:split_importance}. We note that for any $x$, we always have $\|A x\|^p_p=\|A' x\|^p_p$ and $|(a'_i)^{\top} x|^p=|a_i^{\top} x|^p$ for any $i<n$. These two indicate $u_i=u'_i$ for $i<n$.

Now we prove $u_n=u'_n/k$. For any $x$, we still have $\|A x\|_p^p=\|A' x\|_p^p$ but $|(a'_n)^{\top} x|^p=|a_n^{\top} x|^p/k$. These two indicate $u'_n=u_n/k$.

\subsection{Almost Uniform Lewis Weights}\label{sec:almost_uniform_Lewis}

Without loss of generality, we assume $A^{\top} A=I$. We recall the definition of leverage scores that will be used in this proof. Given $p=2$ and $A$ with rows $a_1,\ldots,a_n$, the leverage score of $a_i$ is $a_i^{\top} (A^{\top} A)^{-1} a_i$. Since we assume  $A^{\top} A = I$, this simplifies the score to $\|a_i\|_2^2$. 

Let us start with the case of the uniform Lewis weights for ease of exposition. Because $w_1=w_2=\cdots=w_n=d/n$, we simplify the equation of the Lewis weights
\[
a_i^{\top} \cdot (A^{\top} W^{-1} A)^{1-2/p} a_i = w_i^{2/p}
\]
to $a_i^{\top} \cdot w_i^{2/p-1} \cdot I \cdot a_i=w_i^{2/p}$. This indicates that the leverage score $\|a_i\|^2_2=w_i=d/n$ is also uniform. 

Then we show the upper bound $\frac{|a_i^{\top} x|^p}{\|A x\|_p^p} \le d/n$. By Cauchy-Schwartz,
\[
|a_i^{\top} x| \le \|a_i\|_2 \cdot \|x\|_2 = \sqrt{d/n} \cdot \|x\|_2.
\]
Next we lower bound $\|A x\|_p^p$ by
\begin{equation}\label{eq:p_power_norm}
\|A x\|_p^p \cdot \|A x\|^{2-p}_{\infty} \ge \|A x\|_2^2.
\end{equation}
Because $A^{\top} A = I$, the right hand side is $\|x\|_2^2$. So 
\[
\|A x\|_p^p \ge \frac{\|x\|_2^2}{(\|x\|_2 \cdot \sqrt{d/n})^{2-p}}=\frac{\|x\|^p_2}{(d/n)^{1-p/2}}.
\]
We combine the upper bound and lower bound to obtain
\[
\frac{|a_i^{\top} x|^p}{\|A x\|_p^p} \le \frac{(\sqrt{d/n} \cdot \|x\|_2)^p}{\|x\|^p_2/(d/n)^{1-p/2}}=d/n.
\]
The lower bound $\frac{|a_i^{\top} x|^p}{\|A x\|_p^p} \ge \frac{d^{p/2}}{n}$ follows from choosing $x=a_i$ and replacing \eqref{eq:p_power_norm} by the Holder's inequality $\|Ax\|_p^p \le (n)^{\frac{2-p}{2}} \cdot \|Ax\|_2^{p}$. 

In the non-uniform case, we will use the stability of Lewis weights (Definition 5.1 and 5.2 in \cite{CP15_Lewis}) to finish the proof. Consider any $\ov{W}=diag[\ov{w}_1,\ldots,\ov{w}_n]$ satisfies
\[
\forall i \in [n], a_i^{\top} \left( A^{\top} \ov{W}^{1-2/p} A\right)^{-1} a_i \approx_\alpha \ov{w}_i^{2/p}.
\]
Lemma 5.3 in \cite{CP15_Lewis} shows that for each $i\in[n]$, $\ov{w}_i \approx_{\alpha^{c_p}} w_i$ for the Lewis weights $(w_1,\ldots,w_n)$ where the constant $c_p=\frac{p/2}{1-|p/2-1|}$.

Back to our problem, if the $\ell_p$ Lewis weights of $A$ are almost uniform, we show the uniformity for its $\ell_q$ Lewis weights.

\begin{claim}\label{clm:non_uniform_Lewis_w}
Given $p$ and $q$ less than $4$, let $A \in \mathbb{R}^{n \times d}$ be a matrix whose $\ell_p$ Lewis weight $(w_1,\ldots,w_n)$ satisfies $w_i \approx_{\alpha} d/n$ for each $i \in [n]$. Then the $\ell_q$ Lewis weight of $A$ satisfies $w'_i \approx_{\alpha^C} d/n$ for constant $C=(4/p-1) \cdot c_q$. In particular, when $p=2$ and $q=1$, $C=1$.
\end{claim}
\begin{proof}
Since $w_i \approx_{\alpha} d/n$, $\alpha^{-1} d/n \cdot I \preceq W \preceq \alpha d/n \cdot I$. We plug this sandwich-bound into $A^{\top} W^{1-2/p} A$, which becomes between $(\alpha d/n)^{1-2/p} \cdot A^{\top} A$ and $(\alpha^{-1} d/n)^{1-2/p} \cdot A^{\top} A$. Similarly, we bound its inverse $(A^{\top} W^{1-2/p} A)^{-1}$ by
\[
(\alpha^{-1} d/n)^{2/p - 1} \cdot (A^{\top} A)^{-1} \preceq (A^{\top} W^{1-2/p} A)^{-1} \preceq (\alpha d/n)^{2/p-1} \cdot (A^{\top} A)^{-1}.
\] From the definition of $\ell_p$ Lewis weights $a_i^{\top} \left( A^{\top} W^{1-2/p} A\right)^{-1} a_i=w_i^{2/p}$, we have
\[
(\alpha^{-1} d/n)^{2/p - 1} \cdot a_i^{\top} (A^{\top} A)^{-1} a_i \le w_i^{2/p} \le (\alpha d/n)^{2/p  - 1} \cdot a_i^{\top} (A^{\top} A)^{-1} a_i.
\]
We combine the lower bound above with the property $w_i \approx_{\alpha} d/n$ to upper bound the leverage score $a_i^{\top} (A^{\top} A)^{-1} a_i$ by
\[
(\alpha \cdot d/n)^{2/p} \cdot (\alpha)^{2/p - 1} \cdot (d/n)^{1 - 2/p}=\alpha^{4/p - 1} \cdot d/n.
\]
Similarly, we lower bound the leverage score $a_i^{\top} (A^{\top} A)^{-1} a_i$ by $\alpha^{1-4/p} \cdot d/n$. These two imply the leverage scores are almost uniform: $a_i^{\top} (A^{\top} A)^{-1} a_i \approx_{\alpha^{4/p-1}} d/n$. So for the $\ell_q$ Lewis weights, we approximate it by $\ov{W}=d/n \cdot I$ and get
\[
a_i^{\top} \left( A^{\top} \ov{W}^{1-2/q} A \right)^{-1} a_i = (d/n)^{2/q-1} \cdot a_i^{\top} (A^{\top} A)^{-1} a_i \approx_{\alpha^{4/p-1}} (d/n)^{2/q}.
\]
From the stability of the Lewis weight, we know the actual $\ell_q$ Lewis weights $(w'_1,\ldots,w'_n)$ satisfy $w'_i \approx_{\alpha^{(4/p-1)\cdot c_q}} d/n$
\end{proof}

Now we restate Lemma~\ref{lem:uniform_lewis_bounds_importance} here.
\begin{lemma}\label{lem:non_uniform_import_p}
Given $p$, let the matrix $A$ satisfy (1) $A^{\top} A=I$ and (2) its Lewis weights $w_i \approx_\alpha d/n$ for each row $i$. We have the $\ell_p$ importance weights of $A$ satisfy
\[
\max_x \frac{|a_i^{\top} x|^p}{\|Ax\|_p^p} \in \left[ \alpha^{-C p/2} \cdot d^{p/2}/n , \alpha^{C} \cdot d/n \right]
\]
for $C=4/p-1$.
\end{lemma}
\begin{proof}
For the upper bound, we have $|a_i^{\top} x| \le \|a_i\|_2 \cdot \|x\|_2$ for all $i$. In particular, this implies $\|Ax\|_{\infty} \le \sqrt{\alpha^{C} \cdot d/n} \cdot \|x\|_2$ for $C=4/p-1$ after plugging $q=2$ into Claim~\ref{clm:non_uniform_Lewis_w} to bound $\|a_i\|_2^2$ by the leverage score. Then the rest of the proof is the same as the uniform case shown in the beginning of this section: We lower bound 
\[
\|Ax\|_p^p \ge \|Ax\|_2^2/\|Ax\|^{2-p}_{\infty} = \|x\|_2^2/ (\sqrt{\alpha^{C} \cdot d/n} \cdot \|x\|_2)^{2-p}.
\]
Now we upper bound $\max_x \frac{|a_i^{\top} x|^p}{\|Ax\|_p^p} $ by
\[
\frac{(\sqrt{\alpha^{C} \cdot d/n} \cdot \|x\|_2)^p}{\|x\|_2^2/ (\sqrt{\alpha^{C} \cdot d/n} \cdot \|x\|_2)^{2-p}} \le \alpha^{C} \cdot d/n.
\]

For the lower bound, we choose $x=a_i$ such that $|a_i^{\top} x|=\|a_i\|_2 \cdot \|x\|_2$. Then by Holder's inequality, $\|Ax\|_p^p \le (n)^{\frac{2-p}{2}} \cdot \|Ax\|_2^{p}$. These two show
\[
\frac{|a_i^{\top} x|^p}{\|Ax\|_p^p} \ge \frac{(\sqrt{\alpha^{-C} \cdot d/n} \cdot \|x\|_2)^p}{(n)^{\frac{2-p}{2}} \cdot \|x\|_2^{p}} \ge \alpha^{-C \cdot p/2} \cdot d^{p/2}/n.
\]
\end{proof}

\section{Additional Proofs from Section~\ref{sec:analysis_ell1}}\label{sec:proof_thm_deviation}
Since any row with $p_i \ge 1$ will always be selected in $S$, those rows will not affect the random variable $\wt{L}(\beta)$ considered in this section. We restrict our attention to rows with $p_i < 1$ in this proof. As a warm up, we will first prove the following lemma, which is a simplified version of Lemma~\ref{lem:high_prob_guarantee}. 
\begin{lemma}\label{lem:concentration_deviation}
Given $A \in \mathbb{R}^{n \times d}$ and $\epsilon$, let $w_1,\ldots,w_n$ be the Lewis weight of each row $A_i$. Let $p_1,\ldots,p_n$ be a sequence of numbers upper bounding $w$, i.e.,
\[
p_i \ge w_i/u \text{ for } u =\Theta\left( \frac{\eps^2}{\log (m+d/\eps)} \right) \text{ where } m = \sum_i p_i.
\]
Suppose the coefficients $(s_1,\ldots,s_n)$ in $\wt{L}(\beta):=\sum_{i=1}^n s_i \cdot |a_i^{\top} \beta - y_i|$ is generated as follows: For each $i \in [n]$, with probability $p_i$, we set $s_i=1/p_i$ (when $p_i \ge 1$, $s_i$ is always 1). Given any subset $B_{\beta} \subset \mathbb{R}^d$ of $\beta$,
\[
\E \left[ \sup_{\beta \in B_{\beta}} \left\{ \wt{L}(\beta)-\wt{L}(\beta^*)-\big( L(\beta)-L(\beta^*) \big) \right\} \right] \le \eps \cdot \sup_{\beta \in B_{\beta}} \|A (\beta-\beta^*)\|_1.
\]
\end{lemma}
We remark that while Lemma~\ref{lem:concentration_deviation} implies that rescaling of $u$ to $u=u/\delta^2$ gives
\[
\text{ w.p. $1-\delta$}, \wt{L}(\beta)-\wt{L}(\beta^*) = L(\beta)-L(\beta^*) \pm \eps \cdot \sup_{\beta \in B_{\beta}}\|A (\beta-\beta^*)\|_1 \text{ for all } \beta \in B_{\beta},
\] 
Lemma~\ref{lem:high_prob_guarantee} has a better dependency on $\delta$. 

We need a few ingredients about random Gaussian processes to finish the proofs of Lemma~\ref{lem:concentration_deviation} and Lemma~\ref{lem:high_prob_guarantee}. Except the chaining arguments in \cite{CP15_Lewis}, the proof of Lemma~\ref{lem:concentration_deviation} will use the following Gaussian comparison theorem (e.g. Theorem 7.2.11 and Exercise 8.6.4 in \cite{Vershynin}). In the rest of this section, we use $g$ to denote an i.i.d.~Gaussian vector $N(0,1)^d$ when the dimension is clear.
\begin{theorem}[Slepian-Fernique]\label{thm:Slepian_Fernique}
Let $v_0,\ldots,v_n$ and $u_0,\ldots,u_n$ be two sets of vectors in $\mathbb{R}^d$ where $v_0=u_0=\vec{0}$. Suppose that 
\[
\|v_i-v_j\|_2 \ge \|u_i-u_j\|_2 \text{ for all } i,j=0,\ldots,n.
\]
Then $\E_g\left[\max_i \big| \langle v_i, g \rangle \big| \right] \ge C_0 \cdot \E_g \left[\max_i \big| \langle u_i,g \rangle \big| \right]$ for some constant $C_0$.
\end{theorem}
The proof of Lemma~\ref{lem:high_prob_guarantee} follows the same outline except that we will use the following higher moments version for a better dependency on $\delta$.
\begin{corollary}[Corollary 3.17 of \cite{LTbook}]\label{cor:comparison_higher}
Let $v_0,\ldots,v_n$ and $u_0,\ldots,u_n$ be two sets of vectors satisfying the conditions in the above Theorem. Then for any $\ell>0$, 
\[4^{\ell} \cdot \E_g\left[\max_i \big| \langle v_i, g \rangle \big|^{\ell} \right] \ge \E_g \left[\max_i \big| \langle u_i,g \rangle \big|^{\ell} \right].\]
\end{corollary}

We will use the following concentration bound for $\ell_1$ ball when the Lewis weights are uniformly small for all rows. It is an extension of Lemma 8.2 in \cite{CP15_Lewis}; but for completeness, we provide a proof in Section~\ref{sec:ell_1_concentration}.
\begin{lemma}\label{lem:Gaussian_proc_bounded_Lewis}
Let $A$ be a matrix with Lewis weight upper bounded by $u$. For any set $S \subseteq \mathbb{R}^{d}$,
\[
\E_g\left[\max_{\beta \in S} \bigg| \langle g, A \beta \rangle \bigg| \right] \lesssim \sqrt{u \cdot \log n} \cdot \max_{\beta \in S} \|A \beta\|_1.
\]
\end{lemma}

We finish the proof of Lemma~\ref{lem:concentration_deviation} here and defer the proof of Lemma~\ref{lem:high_prob_guarantee} to Section~\ref{sec:proof_cor_high_prob}.

\begin{proofof}{Lemma~\ref{lem:concentration_deviation}}
For convenience, we define 
\[
\delta_i(\beta):=|a_i^{\top} \beta - y_i| \text{ such that } \wt{L}(\beta)-\wt{L}(\beta^*)=\sum_i s_i (\delta_i(\beta)-|y_i|)
\] given the assumption $\beta^*=0$ and rewrite $L(\beta)-L(\beta^*)=\sum_i (\delta_i(\beta)-|y_i|)$ similarly. 

So, we have
\[
\E_{S} \left[ \sup_{\beta \in B_{\beta}} \left\{ \wt{L}(\beta)-\wt{L}(\beta^*)-\big( L(\beta)-L(\beta^*) \big) \right\} \right] = \E_{S}\left[ \sup_{\beta \in B_{\beta}} \bigg|\sum_{j=1}^n s_j \cdot \big( \delta_{j}(\beta) - |y_{j}| \big) - \sum_{i=1}^n \big( \delta_i(\beta) - |y_i| \big) \bigg| \right].
\]
Since each $\big( \delta_i(\beta) - |y_i| \big)$ is the expectation of $s_i \cdot \big( \delta_{i}(\beta) - |y_{i}| \big)$, from the standard symmetrization and Gaussianization in \cite{LTbook} (which is shown in  Section~\ref{sec:sym_gau} for completeness), this is upper bounded by
\[
\sqrt{2 \pi} \E_S \left[ \E_{g} \left[ \sup_{\beta \in B_{\beta}} \bigg|\sum_{j=1}^n g_j \cdot s_j \cdot \big( \delta_{j}(\beta) - |y_{j}|\big) \bigg| \right] \right].
\]

Then we fix $S$ and plan to apply the Gaussian comparison Theorem~\ref{thm:Slepian_Fernique} to the following process.
\[
\E_g \left[ \sup_{\beta \in B_{\beta}} \bigg|\sum_{j} g_j \cdot s_j \cdot \big( \delta_{j}(\beta) - |y_{j}|\big) \bigg| \right] = \E_g \left[ \sup_{\beta \in B_{\beta}} \bigg| \bigg\langle g, \big(s_j \cdot ( \delta_{j}(\beta) - |y_{j}|)\big)_{j \in [n]} \bigg\rangle \bigg| \right].
\]
Let us verify the condition of the Gaussian comparison Theorem~\ref{thm:Slepian_Fernique} to upper bound this by \[C_2 \cdot \E_g \left[ \sup_{\beta \in B_{\beta}} \bigg| \bigg\langle g, SA\beta \bigg\rangle \bigg| \right].\] Note that for any $\beta$ and $\beta'$, the $j$th term of $\beta$ and $\beta'$ in the above Gaussian process is upper bounded by
\[
\left| s_j \cdot \big( \delta_{j}(\beta) - |y_{j}|\big) - s_j \cdot \big( \delta_{j}(\beta') - |y_{j}|\big) \right| \le s_j \cdot \left| \delta_j(\beta)-\delta_j(\beta') \right| \le s_j \cdot \left| a_j^{\top} \beta - a_j^{\top} \beta' \right|.
\]
At the same time, for the $\vec{0}$ vector, we always have
\[
\left| s_j \cdot \big( \delta_{j}(\beta) - |y_{j}|\big) - 0 \right| \le s_j \cdot \left| |a_j^{\top} \beta - y_j| - |y_j| \right| \le s_j \cdot |a_j^{\top} \beta|.
\]
Hence, the Gaussian comparison Theorem~\ref{thm:Slepian_Fernique} implies that the Gaussian process is upper bounded by
\[
\E_g \left[ \sup_{\beta \in B_{\beta}} \bigg|\sum_{j} g_j \cdot s_j \cdot \big( \delta_{j}(\beta) - |y_{j}|\big) \bigg| \right] \le \E_g \left[ \sup_{\beta \in B_{\beta}} \bigg| \langle g, S A \beta \rangle \bigg| \right].
\]
Next we plan to use Lemma~\ref{lem:Gaussian_proc_bounded_Lewis} to bound this Gaussian process on $SA\beta$.

However, Lemma~\ref{lem:Gaussian_proc_bounded_Lewis} requires that $S A$ has bounded Lewis weight. To show this, the starting point is that if we replace $a_i$ in \eqref{eq:def_Lewis_weight} by $s_i \cdot a_i = \frac{1}{p_i} a_i$, its Lewis weight becomes $w_i/p_i \le u$ by our setup. So we only need to handle the inverse in the middle of \eqref{eq:def_Lewis_weight}. The rest of the proof is very similar to the proof of Lemma~7.4 in \cite{CP15_Lewis}: We add a matrix $A'$ into the Gaussian process to have this property of bounded Lewis weight. Recall that $u \le \frac{\eps^2}{C \log (m + d/\eps)}$ for a large constant $C$. By sampling the rows in $A$ with the Lewis weight and scaling every sample by a factor of $u$, we have the existence of $A'$ with the following 3 properties (see Lemma B.1 in \cite{CP15_Lewis} for the whole argument using the Lewis weight),
\begin{enumerate}
\item $A'$ has $\tilde{O}(d/u)$ rows and each row has Lewis weight at most $u$.
\item The Lewis weight $W'$ of $A'$ satisfies $A'^{\top} \ov{W'}^{1-2/p} A' \succeq A^{\top} \ov{W}^{1-2/p} A$.
\item $\|A' x\|_1 = O(\|A x\|_1)$ for all $x$.
\end{enumerate}
Let $A''$ be the union of $SA$ and $A'$ (so the $j$th row of $A''$ is $s_j \cdot a_j$ for $j \le [n]$ and $A'[j,*]$ for $j \ge n+1$). Because the first $n$ entries of $A'' \cdot \beta$ are the same as $SA\beta$ for any $\beta$ and the rest entries will only increase the energy, we have
\[
\E_g \left[ \sup_{\beta \in B_{\beta}} \bigg| \langle g, S A \beta \rangle \bigg| \right] \le \E_g \left[ \sup_{\beta \in B_{\beta}} \bigg| \langle g, A'' \cdot \beta \rangle \bigg| \right].
\]

To apply Lemma~\ref{lem:Gaussian_proc_bounded_Lewis}, let us verify the Lewis weight of $A''$ is bounded. First of all, $A''^{\top} \overline{W}''^{1-2/p} A'' \succeq A^{\top} \overline{W}^{1-2/p} A$ from Lemma 5.6 in \cite{CP15_Lewis}. Now for each row $s_j a_j$ in $A''$, its Lewis weight is upper bounded by
\[
\left( \frac{1}{p_j} a_j^{\top} (A^{\top} \overline{W}^{1-2/p} A)^{-1} \cdot \frac{1}{p_j} a_j \right)^{1/2} \le w_j/p_j \le u
\]
by the assumption $p_j \ge w_j/u$ in the lemma. At the same time, the Lewis weight of $A'$ is already bounded by $u$ from the definition, which indicates the rows in $A''$ added from $A'$ are also bounded by $u$. So Lemma~\ref{lem:Gaussian_proc_bounded_Lewis} implies 
\begin{align*}
\E_g \left[ \sup_{\beta \in B_{\beta}} \bigg| \langle g, A'' \beta \rangle \bigg| \right] & \le \sqrt{u \log (m+d/u)} \cdot \sup_{\beta \in B_{\beta}} \|A'' \beta\|_1 \\
& \le \sqrt{u \cdot \log (m+d/u)} \cdot \left( \sup_{\beta \in B_{\beta}} \|S A \beta\|_1 + O(\sup_{\beta \in B_{\beta}} \|A \beta\|_1) \right)
\end{align*}
by the definition of $A''$ and the last property of $A'$. Finally, we bring the expectation of $S$ back to bound $\E_S[\sup_{\beta \in B_{\beta}} \|S A \beta\|_1] \le (1+\eps)\sup_{\beta \in B_{\beta}} \|A \beta\|_1$ by Lemma 7.4 in \cite{CP15_Lewis} about the $\ell_1$ subspace embedding.

From all discussion above, the deviation \[\E_g \left[ \sup_{\beta \in B_{\beta}} \bigg| \langle g, A'' \beta \rangle \bigg| \right]\] is upper bounded by $O(\eps) \cdot \sup_{\beta \in B_{\beta}} \|A \beta\|_1$ given our choice of $u$.
\end{proofof} 

\subsection{Proof of Lemma~\ref{lem:high_prob_guarantee}}\label{sec:proof_cor_high_prob}
The proof follows the same outline of the proof of Lemma~\ref{lem:concentration_deviation} except using Corollary~\ref{cor:comparison_higher} to replace Theorem~\ref{thm:Slepian_Fernique} with a higher moment. We choose the moment $\ell=O(\log n/\delta)$ and plan to bound 
\[
\E_{S}\left[ \sup_{\beta \in B_{\beta}} \bigg|\sum_{j=1}^n s_j \cdot \big( \delta_{j}(\beta) - |y_{j}| \big) - \sum_{i=1}^n \big( \delta_i(\beta) - |y_i| \big) \bigg|^{\ell} \right]
\] in the rest of this section for a better dependence on the failure probability. We apply symmetrization and Gaussianization again with different constants:
\begin{align*}
& \E_{S}\left[ \sup_{\beta \in B_{\beta}} \bigg|\sum_{j=1}^n s_j \cdot \big( \delta_{j}(\beta) - |y_{j}| \big) - \sum_{i=1}^n \big( \delta_i(\beta) - |y_i| \big) \bigg|^{\ell} \right] \\
= & \E_{S}\left[ \sup_{\beta \in B_{\beta}} \bigg|\sum_{j=1}^n s_j \cdot \big( \delta_{j}(\beta) - |y_{j}| \big) - \E_{S'} \sum_{j=1}^n s'_j \cdot \big( \delta_{j}(\beta) - |y_{j}| \big)\bigg|^{\ell} \right]\\
& (\notag{\text{Use the convexity of $|\cdot|^{\ell}$ to move $\E_{S'}$ out}})\\
\le & \E_{S,S'}\left[ \sup_{\beta \in B_{\beta}} \bigg|\sum_{j=1}^n s_j \cdot \big( \delta_{j}(\beta) - |y_{j}| \big) - \sum_{j=1}^n s'_j \cdot \big( \delta_{j}(\beta) - |y_{j}| \big)\bigg|^{\ell} \right]\\
& (\notag{\text{Use the symmetry of $S$ and $S'$}})\\
\le & \E_{S,S',\eps \in \{\pm 1\}^n}\left[ \sup_{\beta \in B_{\beta}} \bigg|\sum_{j=1}^n \epsilon_j \cdot (s_j-s'_j) \cdot \big( \delta_{j}(\beta) - |y_{j}| \big) \bigg|^{\ell} \right]\\ 
& (\notag{\text{Split $S$ and $S'$}})\\
\le & \E_{S,S',\eps \in \{\pm 1\}^n}\left[ \sup_{\beta \in B_{\beta}} \bigg|\sum_{j=1}^n \epsilon_j \cdot s_j \cdot \big( \delta_{j}(\beta) - |y_{j}| \big) - \sum_{j=1}^n \epsilon_j \cdot s'_j \cdot \big( \delta_{j}(\beta) - |y_{j}| \big)\bigg|^{\ell} \right]\\
& (\notag{\text{Pay an extra $2^{\ell}$ factor to bound the cross terms}})\\
\le & 2^{\ell} \cdot \E_{S,\eps \in \{\pm 1\}^n}\left[ \sup_{\beta \in B_{\beta}} \bigg|\sum_{j=1}^n \epsilon_j \cdot s_j \cdot \big( \delta_{j}(\beta) - |y_{j}| \big) \bigg|^{\ell} \right] + 2^{\ell} \cdot \E_{S',\eps \in \{\pm 1\}^n}\left[ \sup_{\beta \in B_{\beta}} \bigg|\sum_{j=1}^n \epsilon_j \cdot s'_j \cdot \big( \delta_{j}(\beta) - |y_{j}| \big) \bigg|^{\ell} \right]\\
\le & 2^{\ell+1} \cdot \E_{S,\eps \in \{\pm 1\}^n}\left[ \sup_{\beta \in B_{\beta}} \bigg|\sum_{j=1}^n \epsilon_j \cdot s_j \cdot \big( \delta_{j}(\beta) - |y_{j}| \big) \bigg|^{\ell} \right]\\
& (\notag{\text{Gaussianize it}})\\
\le & 2^{\ell+1} \cdot \sqrt{\pi/2}^{\ell} \cdot \E_{S,g}\left[ \sup_{\beta \in B_{\beta}} \bigg|\sum_{j=1}^n g_j \cdot s_j \cdot \big( \delta_{j}(\beta) - |y_{j}| \big) \bigg|^{\ell} \right].
\end{align*}
Then we replace the vectors $\big(s_j \cdot (\delta_j(\beta)-|y_j|)\big)_j$ by $SA\beta$ using the Gaussian comparison Corollary~\ref{cor:comparison_higher} (we omit the verification here because it is the same as the verification in the proof of Lemma~\ref{lem:concentration_deviation}):
\[
C_1^{\ell} \cdot \E_S \E_g \left[ \sup_{\beta}  \bigg| \langle g, SA\beta \rangle \bigg|^{\ell}\right].
\]
The proofs of Lemma~8.4 and Theorem 2.3 in Section 8 of \cite{CP15_Lewis} show that
\[
\E_S \E_g \left[ \sup_{\beta}  \bigg| \langle g, SA\beta \rangle \bigg|^{\ell}\right] \le C_2^{\ell} \cdot \eps^{\ell} \cdot \delta \cdot 
\sup_{\beta \in B_{\beta}} \|A \beta\|_1^{\ell}.
\]
From all discussion above, we have
\[ \E_{S}\left[ \sup_{\beta \in B_{\beta}} \bigg|\sum_{j=1}^n s_j \cdot \big( \delta_{j}(\beta) - |y_{j}| \big) - \sum_{i=1}^n \big( \delta_i(\beta) - |y_i| \big) \bigg|^{\ell} \right] \le (C_1 C_2)^{\ell} \cdot \eps^{\ell} \delta \cdot \sup_{\beta \in B_{\beta}} \|A \beta\|_1^{\ell}.\]

This implies with probability $1-\delta$, the R.H.S. is at most $(C_1 C_2)^{\ell} \cdot \eps^{\ell} \cdot \sup_{\beta \in B_{\beta}} \|A \beta\|_1^{\ell}$. So $\sum_{j=1}^n s_j \cdot \big( \delta_{j}(\beta) - |y_{j}| \big) = \sum_{i=1}^n \big( \delta_i(\beta) - |y_i| \big) \pm (C_1 C_2) \cdot \eps \sup_{\beta \in B_{\beta}} \|A \beta\|_1$ for all $\beta \in B_{\beta}$. Finally, we finish the proof by rescaling $\eps$ by a factor of $C_1 \cdot C_2$.

\subsection{Proof of $\ell_1$ Concentration}\label{sec:ell_1_concentration}
We finish the proof of Lemma~\ref{lem:Gaussian_proc_bounded_Lewis} in this section.

\begin{proofof}{Lemma~\ref{lem:Gaussian_proc_bounded_Lewis}}



We consider the natural inner product induced by the Lewis weights $(w_1,\ldots,w_n)$ and define the weighted projection onto the column space of $A$ with $W=\diag(w_1,\ldots,w_n)$ as $\Pi := A \cdot (A^\top W^{-1} A)^{-1} \cdot A^\top W^{-1}$. It is straightforward to verify $\Pi A= A$. Thus we write $\langle g, A \beta \rangle$ as $\langle g, \Pi A \beta \rangle$. In the rest of this proof, let $\Pi_1,\ldots,\Pi_n$ denote its $n$ columns.

Then we upper bound the inner product in the Gaussian process by $\|\cdot\|_1 \cdot \|\cdot\|_{\infty}$:
\[
\E_g \left[ \max_{\beta \in S} \big| \langle \Pi^{\top} g, A \beta \rangle \big| \right] \le \E_g \left[ \max_{\beta \in S} \{ \|\Pi^{\top} g\|_{\infty} \cdot \|A \beta\|_1 \} \right].
\]

So we further simplify the Gaussian process as
\[
\E_g \left[ \|\Pi^{\top} g\|_{\infty} \right] \cdot \max_{\beta \in S} \|A \beta\|_1 \le \sqrt{2 \log n} \cdot \max_{i \in [n]} \|\Pi_i\|_2 \cdot \max_{\beta \in S} \|A \beta\|_1,
\]
where we observe the entry $i$ of $\Pi^{\top} g$ is a Gaussian variable with variance $\|\Pi_i\|^2_2$ and apply a union bound over $n$ Gaussian variables. In the rest of this proof, we bound $\|\Pi_i\|^2_2$. 
\begin{align*}
& \sum_j (A \cdot (A^\top W^{-1} A)^{-1} \cdot A^\top W^{-1})_{j,i}^2 \\
= & \sum_j (w_i^{-1} \cdot a_i^{\top} \cdot (A^\top W^{-1} A)^{-1} \cdot a_j)^2\\
\le & u \sum_j (w_i^{-1} \cdot a_i^{\top} \cdot (A^\top W^{-1} A)^{-1} \cdot a_j w_j^{-1/2})^2 \tag{$w_j \le u$ from the assumption}\\
\le & u w_i^{-2} \cdot \sum_j a_i^{\top} \cdot (A^\top W^{-1} A)^{-1} \cdot a_j w_j^{-1} a_j^{\top} \cdot (A^\top W^{-1} A)^{-1} \cdot a_i\\
\le & u w_i^{-2} \cdot a_i^{\top} (A^\top W^{-1} A)^{-1} \cdot \left( \sum_j a_j w_j^{-1} a_j^{\top} \right) \cdot (A^\top W^{-1} A)^{-1} \cdot a_i \\
\le & u w_i^{-2} a_i^{\top} (A^\top W^{-1} A)^{-1} \cdot a_i \tag{by the definition of Lewis weights}\\
\le & u w_i^{-2} \cdot w_i^2 = u.
\end{align*}
\end{proofof}

\subsection{Symmetrization and Gaussianization}\label{sec:sym_gau}
We start with a standard symmetrization by replacing $\sum_{i=1}^n \big( \delta_i(\beta) - |y_i| \big)]$ by its expectation $\E_{S'}\bigg[\sum_{j=1}^n s'_j \cdot \big( \delta_{j}(\beta) - |y_{j}| \big)\bigg]$:
\[
\E_{S}\left[ \sup_{\beta \in B_{\beta}} \bigg|\sum_{j=1}^n s_j \cdot \big( \delta_{j}(\beta) - |y_{j}| \big) - \E_{S'}\bigg[\sum_{j=1}^n s'_j \cdot \big( \delta_{j}(\beta) - |y_{j}| \big)\bigg] \bigg| \right].
\]
Using the covexity of the absolute function, we move out the expectation over $S'$ and upper bound this by
\begin{align*}
& \E_{S,S'}\left[ \sup_{\beta \in B_{\beta}} \bigg|\sum_{j=1}^n s_j \cdot \big( \delta_{j}(\beta) - |y_{j}| \big) - \sum_{j=1}^n s'_j \cdot \big( \delta_{j}(\beta) - |y_{j}| \big) \bigg| \right] \\
= & \E_{S,S'}\left[ \sup_{\beta \in B_{\beta}} \bigg|\sum_{j=1}^n s_j \cdot \big( \delta_{j}(\beta) - |y_{j}|\big) - s'_j \cdot \big( \delta_{j}(\beta) - |y_{j}| \big) \bigg| \right].
\end{align*}
Because $S$ and $S'$ are symmetric and each coordinate is independent, this expectation is equivalent to
\begin{align*}
& \E_{S,S',\sigma}\left[ \sup_{\beta \in B_{\beta}} \bigg|\sum_{j=1}^n \sigma_j \cdot s_j \cdot \big( \delta_{j}(\beta) - |y_{j}|\big) - \sigma_j \cdot s'_j \cdot \big( \delta_{j}(\beta) - |y_{j}| \big) \bigg| \right] \\
\le & \E_{S,S',\sigma}\left[ \sup_{\beta \in B_{\beta}} \bigg|\sum_{j=1}^n \sigma_j \cdot s_j \cdot \big( \delta_{j}(\beta) - |y_{j}|\big) \bigg| + \bigg|\sum_{j=1}^n  \sigma_j \cdot s'_j \cdot \big( \delta_{j}(\beta) - |y_{j}| \big) \bigg| \right]\\
\le & 2 \E_{S,\sigma} \left[ \sup_{\beta \in B_{\beta}} \bigg|\sum_{j=1}^n \sigma_j \cdot s_j \cdot \big( \delta_{j}(\beta) - |y_{j}|\big) \bigg| \right].
\end{align*}
Then we apply Gaussianization: Since $\E[|g_j|]=\sqrt{2/\pi}$, the expectation is upper bounded by
\[
\sqrt{2 \pi} \E_{S,\sigma} \left[ \sup_{\beta \in B_{\beta}} \bigg|\sum_{j=1}^n \E[|g_j|] \sigma_j \cdot s_j \cdot \big( \delta_{j}(\beta) - |y_{j}|\big) \bigg| \right].
\]
Using the convexity of the absolute function again, we move out the expectation over $g_j$ and upper bound this by
\[
\sqrt{2 \pi} \E_{S,\sigma,g} \left[ \sup_{\beta \in B_{\beta}} \bigg|\sum_{j=1}^n |g_j| \sigma_j \cdot s_j \cdot \big( \delta_{j}(\beta) - |y_{j}|\big) \bigg| \right].
\]
Now $|g_j| \sigma_j$ is a standard Gaussian random variable, so we simplify it to
\[
\sqrt{2 \pi} \E_{S,g} \left[ \sup_{\beta \in B_{\beta}} \bigg|\sum_{j=1}^n g_j \cdot s_j \cdot \big( \delta_{j}(\beta) - |y_{j}|\big) \bigg| \right].
\]

\section{Additional Proofs from Section~\ref{sec:analysis_ellp}}\label{sec:property_uniform_Lewis}
We finish the proof of Theorem~\ref{thm:property_ellp} in this section. As mentioned in Section~\ref{sec:analysis_ellp}, we state the following lemma for the 2nd step about matrices with almost uniform Lewis weights. It may be convenient to assume $\gamma=O(1)$, $\alpha = O(1)$, $w'_i \approx d/n$ such that the probability $\frac{m \cdot w'_i}{d}<1$ in this statement.

\begin{lemma}\label{lem:convergence_p}
Given any matrix $A$, let $(w'_1,\ldots,w'_n)$ be a $\gamma$-approximation of the Lewis weights $(w_1,\ldots,w_n)$ of $A$, i.e., $w'_i \approx_{\gamma} w_i$ for all $i \in [n]$. Further, suppose the Lewis weights are almost uniform: $w_i  \approx_{\alpha} d/n$ for a given parameter $\alpha$. 

Let $m=O \big( \frac{\alpha^{O(1)} \gamma \cdot d^2 \log d/\eps\delta}{\eps^2} + \frac{\alpha^{O(1)} \gamma \cdot d^{2/p}}{\eps^2 \delta} \big)$. For each $i\in[n]$, we randomly generate $s_i=\frac{d}{m \cdot w'_i}$ with probability $\frac{m \cdot w'_i}{d}$ and $0$ otherwise. Then with probability $1-\delta$, for $\wt{L}(\beta):=\sum_i s_i \cdot |a_i^{\top} \beta - y_i|^p$, we have
\[ \wt{L}(\beta) - \wt{L}(\beta^*) = L(\beta)-L(\beta^*) \pm \eps \cdot L(\beta) \text{ for any $\beta \in \mathbb{R}^d$}. \]
\end{lemma}

Now we are ready to finish the proof of Theorem~\ref{thm:property_ellp}.

\begin{proofof}{Theorem~\ref{thm:property_ellp}}
Because we could always adjust $w'_i$ and $m$ by a constant factor, let $\eps \rightarrow 0$ be a tiny constant such that $N_1=w'_1/\eps,N_2=w'_2/\eps,\ldots,N_n=w'_n/\eps$ are integers. We define $A' \in \mathbb{R}^{N \times d}$ with $N=\sum_i N_i$ rows as
\[
\begin{bmatrix} 
a_1/N_1^{1/p} \\
\vdots \\
a_1/N_1^{1/p} \\
a_2/N_2^{1/p} \\
\vdots \\
a_n/N_n^{1/p}
\end{bmatrix}
\]
where the first $N_1$ rows are of the form $a_1/N_1^{1/p}$, the next $N_2$ rows are of the form $a_2/N_2^{1/p}$, and so on.
Similarly, we define $y' \in \mathbb{R}^{N}$ as \[
\left( \underbrace{y_1/N_1^{1/p},\ldots,y_1/N_1^{1/p}}_{N_1},\underbrace{y_2/N_2^{1/p},\ldots,y_2/N_2^{1/p}}_{N_2},\ldots,\underbrace{y_n/N_n^{1/p},\ldots,y_n/N_n^{1/p}}_{N_n} \right). 
\]
By the definition, $L(\beta)=\|A \beta - y\|_p^p=\|A' \beta - y'\|_p^p$.

Then we consider $\wt{L}(\beta)$. First of all, the Lewis weights of $A'$ are 
\[
\left( \underbrace{w_1/N_1,\ldots,w_1/N_1}_{N_1},\underbrace{w_2/N_2,\ldots,w_2/N_2}_{N_2},\ldots,\underbrace{w_n/N_n,\ldots,w_n/N_n}_{N_n} \right). 
\]
Second, $(w'_1,\ldots,w'_n)$ on $A$ induces weights  \[
\ov{w} = \left( \underbrace{w'_1/N_1,\ldots,w'_1/N_1}_{N_1},\underbrace{w'_2/N_2,\ldots,w'_2/N_2}_{N_2},\ldots,\underbrace{w'_n/N_n,\ldots,w'_n/N_n}_{N_n} \right) \textit{ on } A'. 
\]
Then we consider $(s'_1,\ldots,s'_N)$ generated in the way described in Lemma~\ref{lem:convergence_p}: For $j \in [N]$, $s'_j=\frac{d}{m \cdot \ov{w}_j}$ with probability $\frac{m \cdot \ov{w}_j}{d}$ and $0$ otherwise. Let $i \in [n]$ be the row $a_i/N_i^{1/p}$ corresponding to $a'_j$, i.e., $j \in (N_1+\ldots+N_{i-1},N_1+\ldots+N_i]$. Since $\ov{w}_j=w'_i/N_i$, $s'_j=\frac{d \cdot N_i}{m \cdot w'_i}$ with probability $\frac{m \cdot w'_i}{N_i \cdot d}$. So the contribution of all $j$'s corresponding to $i$ is
\begin{align*}
& \sum_{j=N_1+\ldots+N_{i-1}+1}^{N_1+\ldots+N_{i}} s'_j \cdot | (a'_j)^{\top} \cdot \beta - y'_j|^p \\
= & \sum_{j=N_1+\ldots+N_{i-1}+1}^{N_1+\ldots+N_{i}} 1(s'_j>0) \frac{d \cdot N_i}{m \cdot w'_i} \cdot  |a_i^{\top}/N_i^{1/p} \cdot \beta - y_i/N_i^{1/p}|^p \\
= & \left( \sum_{j} 1(s'_j>0) \right) \cdot \frac{d}{m \cdot w'_i} \cdot |a_i^{\top} \beta - y_i|^p.
\end{align*}
Next the random variable $\left( \sum_{j} 1(s'_j>0) \right)$ generated by $s'_j$ converges to a Poisson random variable with mean $N_i \cdot \frac{m \cdot w'_i}{N_i \cdot d} = \frac{m \cdot w'_i}{d}$ when $\eps \rightarrow 0$ and $N_i \rightarrow +\infty$. So $s_i \sim \frac{d}{m \cdot w'_i} \cdot \poiss(\frac{m \cdot w'_i}{d})$ generated in this theorem will converge to the coefficient $\left( \sum_{j} 1(s'_j>0) \right) \cdot \frac{d}{m \cdot w'_i}$ in the above calculation. This implies that $\sum_{i=1}^n s_i \cdot |a_i^{\top} \beta - y_i|^p$ generated in Lemma~\ref{lem:convergence_p} for $A'$ with weights $\ov{w}$ converges to $\wt{L}(\beta) = \sum_j s'_j \cdot | (a'_j)^{\top} \cdot \beta - y'_j|^p$ generated in this theorem when $\eps\rightarrow 0$. 

From all discussion above, we have the equivalences of $L(\beta)=\|A \beta - y\|_p^p=\|A' \beta - y'\|_p^p$ and $\wt{L}(\beta)=\sum_{i=1}^n s_i \cdot |a_i^{\top} \beta - y_i|^p=\sum_j s'_j \cdot | (a'_j)^{\top} \cdot \beta - y'_j|^p$. Then Lemma~\ref{lem:convergence_p} (with the same $\gamma$ and $\alpha=1$) provides the guarantee of $\wt{L}$ for all $\beta$. 
\end{proofof}


Next we prove Lemma~\ref{lem:convergence_p}. We discuss a few properties and ingredients that will be used in this proof besides Fact~\ref{fact:almost_uniform} and Claim~\ref{clm:inner_product} mentioned earlier. 

The key property of $\beta^*$ that will be used in this proof is that the partial derivative in every direction is 0, i.e., $\partial L(\beta^*)=\vec{0}$. Since $\beta^*$ is assumed to be $\vec{0}$, this is equivalent to
\begin{equation}\label{eq:prop_beta*}
\forall j \in [d], \sum_{i=1}^n p \cdot |y_i|^{p-1} \cdot \sign(y_i) \cdot A_{i,j}=0.
\end{equation}

It will be more convenient to write $L(\beta)-L(\beta^*)$ in Lemma~\ref{lem:convergence_p} as $\sum_{i=1}^n \big( |a_i^{\top} \beta -y_i|^p - |y_i|^p \big)$ (given $\beta^*=0$) and similarly for $\wt{L}(\beta)-\wt{L}(\beta^*)$. Then we will use the following claim to approximately bound $|a_i^{\top} \beta - y_i|^p - |y_i|^p$ in the comparison of $L$ and $\wt{L}$.
\begin{claim}\label{clm:taylor_exp_ellp}
For any real numbers $a$ and $b$, 
\[
|a-b|^p - |a|^p + p \cdot |a|^{p-1} \cdot \sign(a) \cdot b = O(|b|^p).
\]
In particular, this implies that
for any $\beta$, $a_i$, and $y_i$, we always have
\[
|a_i^{\top} \beta - y_i|^p - |y_i|^p + p \cdot |y_i|^{p-1} \cdot \sign(y_i) \cdot a_i^{\top} \beta = O(|a_i^{\top} \beta|^p).
\]
\end{claim}

We will use the following claim about random events associated with $s_1,\ldots,s_n$ to finish the proof. Recall that $s_i$ is generated according to the weight $w'_i$. So one property about $w'_i$ that will be used in this proof is $w'_i \approx_{\gamma} w_i$ and $w_i \approx_{\alpha} d/n$, implying that $w'_i \approx_{\gamma \alpha} d/n$. The claim considers the concentration for the approximation in terms of Claim~\ref{clm:taylor_exp_ellp}.
\begin{claim}\label{clm:concentration_single_beta}
Given any $\beta$, with probability at least $1-exp\big( -\Omega(\frac{\eps^2 m}{\alpha^{O(1)} \cdot \gamma \cdot d}) \big)$ (over $s_1,\ldots,s_m$),
\begin{align*}
& \sum_{i=1}^n s_i \cdot \left( |a_i^{\top} \beta - y_i|^p - |y_i|^p + p \cdot |y_i|^{p-1} \cdot \sign(y_i) \cdot a_i^{\top} \beta \right) \\
& = \sum_{i=1}^n \left( |a_i^{\top} \beta - y_i|^p - |y_i|^p + p \cdot |y_i|^{p-1} \cdot \sign(y_i) \cdot a_i^{\top} \beta \right) \pm \eps \|A \beta\|_p^p.
\end{align*}
\end{claim}
Note that the extra term on the R.H.S., $\sum_i p \cdot |y_i|^{p-1} \cdot \sign(y_i) \cdot a_i^{\top} \beta$, is always zero for any $\beta$ from \eqref{eq:prop_beta*}. So the last piece in this proof is to bound its counterpart on the left hand side by Claim~\ref{clm:inner_product} mentioned earlier.

We defer the proof of Claim~\ref{clm:taylor_exp_ellp} to Section~\ref{sec:proof_taylor_exp}, the proof of Claim~\ref{clm:concentration_single_beta} to Section~\ref{sec:proof_single_beta}, and the proof of Claim~\ref{clm:inner_product} to Section~\ref{sec:proof_inner_product} separately.

\begin{proofof}{Lemma~\ref{lem:convergence_p}}
Recall that we assume $\beta^*=0$ in this proof. Let $C_0=\Theta(1/\epsilon\delta)$ with a sufficiently large constant, $\eps'=\eps \cdot \delta/50$ and $B_{\eps'}$ be an $\eps'$-net (in $\ell_p$ norm) of the $\ell_p$ ball with radius $C_0 \|y\|_p$ such that for any $\beta$ with $\|A \beta\|_p \le C_0 \|y\|_p$, $\exists \beta' \in B_{\eps'}$ satisfies $\|A (\beta-\beta')\|_p \le \eps' \|y\|_p$. We set $m=O \big( \frac{\alpha^{O(1)} \gamma \cdot d^2 \log C_0/\eps'}{\eps^2} + \frac{\alpha^{O(1)} \gamma \cdot d^{2/p}}{\eps^2 \delta} \big)$ such that we have the following properties of $s_1,\ldots,s_n$:
\begin{enumerate}
\item Claim~\ref{clm:concentration_single_beta} holds for all $\beta \in B_{\eps'}$. 
\item For all $\beta$, we always have $\sum_i |a_i^{\top} \beta|^p \approx_{1+\epsilon} \sum_i s_i \cdot |a_i^{\top} \beta|^p $. In other words, the subspace embedding $\|S A \beta\|_p^p \approx_{1+\eps} \|A \beta\|_p^p$ for all $\beta$.
\item By Claim~\ref{clm:inner_product}, with probability $1-\delta$, the R.H.S. of Equation \eqref{eq:bound_inner} is $\le \eps \cdot \|A \beta\|_p \cdot \|y\|_p^{p-1}$, which is upper bounded $\eps (\|A \beta\|_p^p + \|y\|_p^p)$.
\item Moreover, since $\E_s[\wt{L}(\beta^*)]=L(\beta^*)$, we assume $\wt{L}(\beta^*):=\|Sy\|_p^p = \sum_{i=1}^n s_i |y_i|^p$ is $\le \frac{1}{\delta} \|y\|_p^p$ w.p.~$1-\delta$ by the Markov's inequality.
\end{enumerate} 

First, we argue the conclusion $\wt{L}(\beta) - \wt{L}(\beta^*) = L(\beta)-L(\beta^*) \pm O(\eps) \cdot (\|A \beta\|_p^p + \|y\|_p^{p})$ holds for all $\beta \in B_{\eps'}$. As mentioned earlier, we rewrite $L(\beta)-L(\beta^*)$ as 
\begin{equation}\label{eq:lhs_cross}
\sum_{i=1}^n \left( |a_i^{\top} \beta - y_i|^p - |y_i|^p + p \cdot |y_i|^{p-1} \cdot \sign(y_i) \cdot a_i^{\top} \beta \right)
\end{equation}
since the cross term is zero from \eqref{eq:prop_beta*}. Then we rewrite $\wt{L}(\beta) - \wt{L}(\beta^*)$ as
\[
\underbrace{\sum_{i=1}^n s_i \cdot \left( |a_i^{\top} \beta - y_i|^p - |y_i|^p + p \cdot |y_i|^{p-1} \cdot \sign(y_i) \cdot a_i^{\top} \beta \right)}_{T_1} - \underbrace{\sum_{i=1}^n s_i \cdot p \cdot |y_i|^{p-1} \cdot \sign(y_i) \cdot a_i^{\top} \beta}_{T_2}.
\]
By Claim~\ref{clm:concentration_single_beta}, the first term $T_1$ is $\eps \cdot \|A \beta\|_p^p$-close to \eqref{eq:lhs_cross}. Moreover, by Property 3 from Claim~\ref{clm:inner_product} mentioned above, the second term $T_2$ is always upper bounded by $\eps (\|y\|_p^p + \|A \beta\|_p^p)$.

Second, we argue $\wt{L}(\beta) - \wt{L}(\beta^*) = L(\beta)-L(\beta^*) \pm O(\eps+\eps') \cdot \big( \|A \beta\|_p^p + \|y\|_p^p \big)$ holds for all $\beta$ with $\|A \beta\|_p \le C_0 \|y\|_p$. Let us fix such a $\beta$ and define $\beta'$ to be the closest vector to it in $B_{\eps'}$ with $\|A \beta - A \beta'\|_p \le \eps' \|y\|_p$. We rewrite 

\begin{equation}\label{eq:net_p_norm}
L(\beta) - L(\beta^*)=\|A \beta' + A (\beta-\beta') - y \|_p^p - L(\beta^*).
\end{equation}

 By triangle inequality, we bound 
\[\|A \beta' + A (\beta-\beta') - y \|_p \in \left[ \|A \beta' - y\|_p - \|A(\beta-\beta')\|_p, \|A \beta' - y\|_p + \|A(\beta-\beta')\|_p \right]
\] and use the approximation in Claim~\ref{clm:taylor_exp_ellp} with $a=\|A \beta' - y\|_p$ and $b=\|A (\beta-\beta')\|_p$ to approximate its $p$th power in \eqref{eq:net_p_norm} as
\begin{align*}
& L(\beta) - L(\beta^*)\\
= & \|A \beta' - y\|_p^p - L(\beta^*) \pm O\left(p \cdot \|A \beta'-y\|_p^{p-1} \cdot \|A (\beta-\beta')\|_p + \|A (\beta-\beta')\|_p^p\right)\\
= & \|A \beta' - y\|_p^p - L(\beta^*) \pm O\left( p \cdot \|A \beta'-y\|_p^{p-1} \cdot \eps' \|y\|_p + \eps'^p \cdot \|y\|_p^p \right).\\
= & \|A \beta' - y\|_p^p - L(\beta^*) \pm \eps' \cdot O\left( \|A \beta'\|_p^{p-1} \cdot \|y\|_p + \|y\|_p^p + \|y\|_p^p \right) \tag{by triangle inequality}\\
= & L(\beta')-L(\beta^*) \pm \eps' \cdot O(\|A \beta\|_p^{p-1} \cdot \|y\|_p + \|y\|_p^p).
\end{align*}
Now we approximate $\wt{L}(\beta) - \wt{L}(\beta^*)$ by defining a weighted $L_p$ norm $\|x\|_{s,p}$ for a vector $x \in \mathbb{R}^n$ as $(\sum_i s_i |x_i|^p)^{1/p}$. Note that $\|\cdot\|_{s,p}$ satisfies the triangle inequality. At the same time, $\wt{L}(\beta)=\|A \beta - y\|_{s,p}$ by the definition and $\|A \beta\|_p^p \approx_{1+\epsilon} \|A \beta\|_{s,p}^p$ by the 2nd property mentioned in this proof. By the same argument as above, we have
\[
\wt{L}(\beta)-\wt{L}(\beta^*)=\wt{L}(\beta') - \wt{L}(\beta^*) \pm \eps' \cdot O(\|A \beta\|_{s,p}^{p-1} \cdot \|y\|_{s,p} + \|y\|_{s,p}^p).
\]
Since $\beta'$ has $\wt{L}(\beta') - \wt{L}(\beta^*) = L(\beta')-L(\beta^*) \pm O(\eps)\cdot (\|A \beta'\|_p^p +\|y\|^{p}_p)$, from all the discussion above, the error between $\wt{L}(\beta) - \wt{L}(\beta^*)$ and $L(\beta)-L(\beta^*)$ is at most
\begin{align*}
& O(\eps)\cdot (\|A \beta'\|_p^p + \|y\|^{p}_p) + \eps' \cdot O(\|A \beta\|_{s,p}^{p-1} \cdot \|y\|_{s,p} + \|y\|_{s,p}^p + \|A \beta\|_p^{p-1} \cdot \|y\|_p + \|y\|_p^p)\\
= &  O(\eps)\cdot (\|A \beta'\|_p^p + \|y\|^{p}_p) + \eps'/\delta \cdot O(\|A \beta\|_{p}^{p-1} \cdot \|y\|_{p} + \|y\|_{p}^p)
\end{align*}
where we replace $\|y\|^p_{s,p}=\|S y\|_p^p$ by $\frac{1}{\delta} \|y\|_p^p$ and apply the subspace embedding Property 2 to $\|A \beta\|_{s,p}$. Given $\eps'=O(\eps \delta)$, the error is at most $\eps \cdot O\big( \|A \beta\|_p^p + \|y\|_p^p \big)$


The last case of $\beta$ is $\|A \beta\|_p > C_0 \|y\|_p$, where we bound $\wt{L}(\beta) - \wt{L}(\beta^*) = L(\beta)-L(\beta^*) \pm O(\eps) \cdot \|A \beta\|_p^p$. Note that $L(\beta):=\|A \beta - y\|_p^p$ is in $(\|A \beta\|_p \pm \|y\|_p)^p$ by the triangle inequality. From the approximation in Claim~\ref{clm:taylor_exp_ellp}, this is about
\[
\|A \beta\|_p^p \pm p \|A \beta\|_p^{p-1} \cdot \|y\|_p \pm O(\|y\|_p^p),
\]
which bounds 
\[
L(\beta)-L(\beta^*)=\|A \beta\|_p^p \pm p \|A \beta\|_p^{p-1} \cdot \|y\|_p \pm O(\|y\|_p^p)=\|A \beta\|_p^p \pm O(1/C_0) \|A \beta\|_p^p.
\] 

Similarly, we have $\wt{L}(\beta) - \wt{L}(\beta^*)=\|A \beta\|_{s,p}^p \pm p \|A \beta\|_{s,p}^{p-1} \cdot \|y\|_{s,p} \pm O(\|y\|_{s,p}^p)$ for the weighted $\ell_p$ norm defined as above. Given $\|y\|_{s,p}^p \le \frac{1}{\delta} \|y\|_p^p$ and $\|A \beta\|_{s,p}^p = (1 \pm \epsilon) \|A \beta\|_p^p$, so
\[
\wt{L}(\beta) - \wt{L}(\beta^*)= \|A \beta\|^p_p \pm \epsilon \|A\beta\|_p^p \pm O(\frac{1}{\delta C_0}) \|A \beta\|_p^p.
\]
When $C_0 = \Theta(1/\eps\delta)$ for a sufficiently large constant, the error is $O(\epsilon) \cdot \|A \beta\|_p^p$.

Finally we conclude that the error is always bounded by $\epsilon \cdot L(\beta):=\epsilon \cdot \|A \beta - y\|_p^p$. 
\begin{enumerate}
\item When $\|A \beta\|_p \le 2\|y\|_p$, $L(\beta) \ge \|y\|_p^p$ by the definition of $\beta^*$. So $\epsilon \cdot O(\|y\|_p^p + \|A \beta\|_p^p)=O(\epsilon) \cdot L(\beta)$. 
\item Otherwise $\|A \beta\|_p > 2 \|y\|_p$. Then $L(\beta) \ge (\|A \beta\|_p - \|y\|_p)^p \ge \|A \beta\|_p^p/4$, we also have $\epsilon \cdot (\|y\|_p^p + \|A \beta\|_p^p)=O(\epsilon) \cdot L(\beta)$.
\end{enumerate}
\end{proofof}


\subsection{Proof of Claim~\ref{clm:taylor_exp_ellp}}\label{sec:proof_taylor_exp}
Without loss of generality, we assume $a>0$. We consider two cases: 
\begin{enumerate}
\item $|b|<a/2$: We rewrite $|a-b|^p - |a|^p + p \cdot |a|^{p-1} \cdot b$ as $\int_{a-b}^{a} p \cdot a^{p-1} - p x^{p-1} \mathrm{d} x$. Now we bound $p \cdot a^{p-1} - p x^{p-1}$ using Taylor's theorem, whose reminder is $f'(\zeta) \cdot (a-x)$ for $f(t)=p \cdot t^{p-1}$ and some $\zeta \in (x,a)$. So we upper bound the integration by
\[
\int_{a-b}^{a} p \cdot a^{p-1} - p x^{p-1} \mathrm{d} x \le \int_{a-b}^a p (p-1) \cdot |\zeta_x|^{p-2} \cdot |a-x| \mathrm{d} x.
\]
Since $p \in (1,2]$, $x \in (a-b,a)$ and $\zeta_x \in (x,a)$, we always upper bound $|\zeta_x|^{p-2}$ by $\frac{1}{|a/2|^{2-p}}$. So this integration is upper bounded by $\frac{p(p-1) \cdot b^2}{2|a/2|^{2-p}}=O(|b|^{p})$ given $|b|<a/2$.
\item $|b| \ge a/2$: We upper bound 
\begin{align*}
|a-b|^p - |a|^p + p \cdot |a|^{p-1} \cdot b \le & (|a|+|b|)^p + |a|^p + p \cdot |a|^{p-1} \cdot |b| \\
\le & (3|b|)^p + (2|b|)^p + p \cdot 2^{p-1} \cdot |b|^p = O(|b|^p).
\end{align*}
\end{enumerate}

\subsection{Proof of Claim~\ref{clm:concentration_single_beta}}\label{sec:proof_single_beta}
Since $s_i=\frac{d}{m \cdot w'_i}$ with probability $\frac{m \cdot w'_i}{d}$ (otherwise $0$), $\E[s_i]=1$ and
\[
\E \sum_{i=1}^n s_i \left( |a_i^{\top} \beta - y_i|^p - |y_i|^p + p \cdot |y_i|^{p-1} \sign(y_i) \cdot a_i^{\top} \beta \right)= \sum_{i=1}^n \left( |a_i^{\top} \beta - y_i|^p - |y_i|^p + p \cdot |y_i|^{p-1} \sign(y_i) a_i^{\top} \beta \right).
\]
To bound the deviation, we plan to use the following version of Bernstein's inequality for $X_i=(s_i-1) \cdot \left( |a_i^{\top} \beta - y_i|^p - |y_i|^p + p \cdot |y_i|^{p-1} \cdot \sign(y_i) \cdot a_i^{\top} \beta \right)$:
\[  \Pr\Big(\sum_i X_i\geq t\Big)\leq\exp\bigg(-\frac{\frac12
    t^2}{\sum_i\E[X_i^2] + \frac13Mt}\bigg).\]
First of all, we bound $M=\sup |X_i|$ as
\begin{align*}
& \frac{d}{m \cdot w'_i} \left| |a_i^{\top} \beta - y_i|^p - |y_i|^p + p \cdot |y_i|^{p-1} \cdot \sign(y_i) \cdot a_i^{\top} \beta \right|\\
\le & \frac{d}{m \cdot w'_i} \cdot C |a_i^{\top} \beta|^{p} \tag{by Claim~\ref{clm:taylor_exp_ellp}}\\
\le & \frac{d}{m} \cdot C \cdot \frac{|a_i^{\top} \beta|^{p}}{w'_i}\\
\le & \frac{d}{m} \cdot C \gamma \cdot \frac{|a_i^{\top} \beta|^{p}}{w_i} \tag{recall $w'_i \approx_{\gamma} w_i$}\\
\le & \frac{d}{m} \cdot C \gamma \cdot \|A \beta\|^{p}_p \tag{by Property~\eqref{eq:importance_ellp}}
\end{align*}
Next we bound $\sum_i \E[X_i^2]$ as
\begin{align*}
& \sum_{i=1}^n \E\left[ (s_i-1)^2 \cdot (a_i^{\top} \beta - y_i|^p - |y_i|^p + p \cdot |y_i|^{p-1} \cdot \sign(y_i) \cdot a_i^{\top} \beta )^2\right]\\
\le & \sum_{i=1}^n \E[(s_i-1)^2] \cdot C^2 \cdot |a_i^{\top} \beta|^{2p} \tag{by Claim~\ref{clm:taylor_exp_ellp}}\\
\le & \sum_{i=1}^n (1+\frac{d}{m \cdot w'_i}) \cdot C^2 \cdot (\frac{\alpha \cdot d}{n} \|A \beta\|_p^p) \cdot |a_i^{\top} \beta|^{p} \tag{by Property~\eqref{eq:importance_ellp}}\\
\le & \frac{2 d \cdot \alpha \gamma \cdot n}{m \cdot d} \cdot C^2 (\frac{\alpha \cdot d}{n} \|A \beta\|_p^p) \cdot \sum_{i=1}^n |a_i^{\top} \beta|^{p} \tag{by $w' \approx_{\alpha \gamma} d/n$}\\
= & \frac{2 C^2 \cdot \alpha^{2} \gamma \cdot d}{m} \|A \beta\|_p^{2p}.
\end{align*}
So for $m=\Omega(\alpha^{2} \gamma d/\eps^2)$ and $t=\eps \|A \beta\|_p^p$, we have the deviation is at most $t$ with probability $1 - exp(-\Omega(\frac{\eps^2 m}{\alpha^{2} \gamma \cdot d}))$.

\subsection{Proof of Claim~\ref{clm:inner_product}}\label{sec:proof_inner_product}
By Cauchy-Schwartz, we upper bound 
\begin{align*}
& \sum_{i=1}^n s_i p \cdot |y_i|^{p-1} \cdot \sign(y_i) \cdot \sum_j A_{i,j} \beta_j \\
= & \sum_j \beta_j \cdot (\sum_i s_i p \cdot |y_i|^{p-1} \cdot \sign(y_i) \cdot A_{i,j}) \\
\le & \left\| \left( \sum_i s_i p \cdot |y_i|^{p-1} \cdot \sign(y_i) \cdot A_{i,j} \right)_{j \in [d]} \right\|_2 \cdot \|\beta\|_2.
\end{align*}
Now we bound the $\ell_2$ norm of those two vectors separately.
\begin{fact}
$\E \left\| \left( \sum_i s_i p \cdot |y_i|^{p-1} \cdot \sign(y_i) \cdot A_{i,j} \right)_{j \in [d]} \right\|^2_2 \le \frac{2\alpha d}{m} \cdot \sum_i |y_i|^{2p-2}$.
\end{fact}
\begin{proof}
We rewrite the $\ell_2$ norm square as
\begin{align*}
& \E \left\| \left( \sum_i s_i p \cdot |y_i|^{p-1} \cdot \sign(y_i) \cdot A_{i,j} \right)_{j \in [d]} \right\|^2_2 - \sum_j \E \left[ \sum_i p \cdot |y_i|^{p-1} \cdot \sign(y_i) \cdot A_{i,j} \right]^2 \tag{the make up term is always zero by Property~\eqref{eq:prop_beta*}}\\
= & \sum_j \E \left[ \left( \sum_i (s_i-1) p \cdot |y_i|^{p-1} \cdot \sign(y_i) \cdot A_{i,j} \right)^2 \right] \\
= & \sum_j \sum_i \E \left( (s_i-1) p \cdot |y_i|^{p-1} \cdot \sign(y_i) \cdot A_{i,j} \right)^2 \tag{use $\E[s_i]=1$ and the independence between different $s_i$'s}\\
\le & \sum_j \sum_i (\frac{d}{m \cdot w'_i}+1) \cdot |y_i|^{2p-2} \cdot A_{i,j}^2 \\
\le & \sum_i |y_i|^{2p-2} \cdot (\frac{d}{m \cdot w'_i}+1) \cdot \sum_j A_{i,j}^2 \tag{use the definition of leverage score in \eqref{eq:leverage_score_ell2} and $w'_i \approx_{\alpha \gamma} d/n$}\\
\le & \sum_i |y_i|^{2p-2} \cdot \frac{2\alpha \gamma n}{m} \cdot \frac{\alpha^{C_p} d}{n} \\
= & \frac{2\alpha^{1+C_p} \gamma d}{m} \cdot \sum_i |y_i|^{2p-2}.
\end{align*}
\end{proof}

\begin{fact}
Suppose $A^{\top} A=I_d$ and its leverage score is almost uniform: $\|a_i\|_2^2 \approx_{\alpha} d/n$. Then
$\|\beta\|_2 \le \|A \beta\|_p \cdot (\frac{\alpha^{C_p} \cdot d}{n})^{\frac{2-p}{2p}}$.
\end{fact}
\begin{proof}
Since $A^{\top} A=I_d$, $\|\beta\|_2=\|A\beta\|_2$. Then we upper bound $\|A \beta\|^2_2$:
\begin{align*}
\|A \beta\|_2^2 & = \sum_{i=1}^n |a_i^{\top} \beta|^2\\
& \le (\sum_i |a_i^{\top} \beta|^p) \cdot \max_i |a_i^{\top} \beta|^{2-p}\\
& \le \|A \beta\|_p^p \cdot (\frac{\alpha \cdot d}{n} \cdot \|A \beta\|_p^p)^{\frac{2-p}{p}} \tag{by Property~\eqref{eq:importance_ellp}}\\
& \le \|A \beta\|_p^2 \cdot (\frac{\alpha \cdot d}{n})^{\frac{2-p}{p}}.
\end{align*}
\end{proof}

Now we are ready to finish the proof. From the 1st fact, with probability $1-\delta$,
\[
\left\| \left( \sum_i s_i p \cdot |y_i|^{p-1} \cdot \sign(y_i) \cdot A_{i,j} \right)_{j \in [d]} \right\|_2 \le 10 \sqrt{\frac{2\alpha^{O(1)} \gamma d}{\delta m} \cdot \|y\|_{2p-2}^{2p-2}}.
\]
From the 2nd fact, we always have
\[
\|\beta\|_2 \le \|A \beta\|_p \cdot (\frac{\alpha \cdot d}{n})^{\frac{2-p}{2p}}.
\]
So with probability $1-\delta$,
\begin{align*}
\sum_{i=1}^n s_i p \cdot |y_i|^{p-1} \cdot \sign(y_i) \cdot \sum_j A_{i,j} \beta_j \le & \left\| \left( \sum_i s_i p \cdot |y_i|^{p-1} \cdot \sign(y_i) \cdot A_{i,j} \right)_{j \in [d]} \right\|_2 \cdot \|\beta\|_2\\
\le & 20 \sqrt{\frac{\alpha^{O(1)} \gamma d}{\delta m}} \cdot \|y\|_{2p-2}^{p-1} \cdot \|A \beta\|_p \cdot (\frac{\alpha \cdot d}{n})^{\frac{2-p}{2p}}\\
\le & C \sqrt{\frac{\alpha^{O(1)} \gamma d}{\delta m}} \cdot \|A \beta\|_p \cdot (\alpha \cdot d)^{\frac{2-p}{2p}} \cdot \|y\|_{2p-2}^{p-1} \cdot (1/n)^{\frac{2-p}{2p}}.
\end{align*}
Finally we use Holder's inequality to bound the last two terms $\|y\|_{2p-2}^{p-1} \cdot (1/n)^{\frac{2-p}{2p}}$ in the above calculation. We set $q_1=\frac{p}{2-p}$ and $q_2=\frac{p}{2p-2}$ (such that $1=1/q_1+1/q_2$) to obtain
\[
(\sum_i |y_i|^{2p-2}) \le (\sum_i 1)^{1/q_1} \cdot (\sum_i |y_i|^{(2p-2) \cdot q_2})^{1/q_2} = n^{\frac{2-p}{p}} \cdot \|y\|_p^{2p-2}.
\]
So we further simplify the above calculation
\[
C \sqrt{\frac{\alpha^{O(1)} \gamma d}{\delta m}} \cdot \|A \beta\|_p \cdot (\alpha \cdot d)^{\frac{2-p}{2p}} \cdot \|y\|_{2p-2}^{p-1} \cdot (1/n)^{\frac{2-p}{2p}}\le C \sqrt{\frac{\alpha^{{O(1)}} \cdot \gamma \cdot d^{2/p}}{\delta m}} \cdot \|A \beta\|_p \cdot \|y\|_p^{p-1}.
\]

\section{Additional Proofs from Section~\ref{sec:lower_bound}}\label{app:proof_lower}
We finish the proof of Theorem~\ref{thm:information_lower_bound} in this section.

\begin{proofof}{Theorem~\ref{thm:information_lower_bound}}
For contradiction, we assume there is an algorithm $P$ that outputs an $(1+\eps/200)$-approximation with probability $1-\delta$ using $m$ queries on $y$. We first demonstrate $m = \Omega(\frac{\log 1/\delta}{\eps^2})$. We pick $\alpha \in \{\pm 1\}$ and generate $y$ as follows:
\begin{enumerate}
\item $(y_1,\ldots,y_{n/d})$ are generated from $D_{\alpha}$ defined in Section~\ref{sec:lower_bound}.
\item The remaining entries are 0, i.e. $y_i=0$ for $i>n/d$.
\end{enumerate}
So $\beta^*=(\alpha,0,\ldots,0)$ with probability $1-\delta/d$. For any $(1+\eps/200)$-approximation $\wt{\beta}$, its first entry has $\sign(\wt{\beta}_1) = \alpha$ from Claim~\ref{clm:distinguish_two_dist}. By the lower bound in Lemma~\ref{lem:information_lower_dim_1}, the algorithm must make $\Omega(\frac{\log 1/\delta}{\eps^2})$ queries to $y_1,\ldots,y_{n/d}$.

Then we show $m=\Omega(d/\eps^2)$. For convenience, the rest of the proof considers a fixed $\delta=0.01$. By Yao's minmax principle, we consider a deterministic algorithm $P$ in this part. Now let us define the distribution $D_b$ over $\{\pm 1\}^n$ for $b \sim \{\pm 1\}^d$ as follows:
\begin{enumerate}
\item We sample $b \sim \{ \pm 1\}^d$.
\item We generate $y \sim D_b$ where $D_b = D_{b_1} \circ D_{b_2} \circ \cdots \circ D_{b_d}$.
\end{enumerate}
For any $b$, when $n>100d \log d/\eps^2$ and $n'=n/d$, with probability $0.99$, for each $i \in [d]$, $D_{b_i}$ will generate $n'$ bits where the number of bits equaling $b_i$ is in the range of $[(1/2+\eps/2)n',(1/2+3\eps/2)n']$. We assume this condition in the rest of this proof. From Claim~\ref{clm:distinguish_two_dist}, $\beta^*$ minimizing $\|A \beta - y\|_1$ will have $b_i=\sign(\beta^*_i)$ for every $i \in [d]$. 

Next, given $A$ and $y$, let $\tilde{\beta}$ be the output of $P$. We define $b'$ as $b'_i=\sign(\tilde{\beta}_i)$ for each coordinate $i$.
We show that $b'$ will agree with $b$ (the string used to generate $y$) on 0.99 fraction of bits when $P$ outputs an $(1+\eps/200)$-approximation. As discussed before, $\|A \tilde{\beta}-y\|_1$ is the summation of $d$ subproblems for $d$ coordinates separately, i.e., $L_1(\tilde{\beta}_1),\ldots,L_d(\tilde{\beta}_d)$. In particular, $\|A \tilde{\beta}-y\|_1 \le (1+\epsilon/200)\|A \beta^*-y\|_1$ implies
\[
\sum_{i=1}^d L_i(\tilde{\beta}_i) \le (1+\eps/200) \sum_{i=1}^d L_i(\beta^*_i).
\]
At the same time, we know $L_i(\tilde{\beta}_i) \ge L_i(\beta^*_i)$ and $L_i(\beta^*_i) \in [(1-\eps)n',(1-3\eps)n']$ for any $i$. This implies that for at least $0.99$ fraction of $i \in [d]$, $L_i(\tilde{\beta}_i) \le (1+\eps) L_i(\beta^*_i)$: Otherwise the approximation ratio is not $(1+\eps/200)$ given
\begin{equation}
0.99 \cdot (1-\eps)+0.01 \cdot (1-3\eps) \cdot (1+\eps) > (1+\eps/200) \cdot \big( 0.99 \cdot (1-\eps)+0.01(1-3\eps) \big).
\end{equation}
 From Claim~\ref{clm:distinguish_two_dist}, for such an $i$, $\tilde{\beta}_i$ will have the same sign with $\beta^*_i$. So $b'$ agree with $b$ on at least $0.99$ fraction of coordinates. 

For each $b$, let $m_i(b)$ denote the expected queries of $P$ on $y_{(i-1)(n/d)+1},\ldots,y_{i(n/d)}$ (over the randomness of $y \sim D_b$). Since $P$ makes at most $m$ queries, we have $\E_y[\sum_i m_i(b)] \le m$.

Now for each $b \in \{\pm 1\}^d$ and coordinate $i \in [d]$, we say that the coordinate $i$ is good in $b$ if (1) $\E_{y}[m_i(b)] \le 60m/d$; and (2) $b'_i=b_i$ with probability 0.8 over $y \sim D_b$ when $b'=\sign(\tilde{\beta})$ defined from the output $\tilde{\beta}$ of $P$. 

Let $b^{(i)}$ denote the flip of $b$ on coordinate $i$. We plan to show the existence of $b$ and a coordinate $i \in [d]$ such that $i$ is good in both $b$ and $b^{(i)}$: Let us consider the graph $G$ on the Boolean cube that corresponds to $b \in \{ \pm 1\}^d$ and has edges of $(b, b^{(i)})$ for every $b$ and $i$. To show the existence, for each edge $(b,b^{(i)})$ in the Boolean cube, we remove it if $i$ is not good in $b$ or $b^{(i)}$. We prove that after removing all bad events by a union bound, $G$ still has edges inside. 

For the 1st condition (1) $\E[m_i(b)] \le 60m/d$, we will remove at most $d/60$ edges from each vertex $b$. So the fraction of edges removed by this condition is at most $1/30$. For the 2nd condition, from the guarantee of $P$, we know
\[
\Pr_{b,y}[\tilde{\beta} \text{ is an $(1+\eps/200)$ approximation}] \ge 0.99.
\]
So for $0.8$ fraction of points $b$, we have $\Pr_{y \sim D_b}[\tilde{\beta} \text{ is an $(1+\eps/200)$ approximation}] \ge 0.95$ (Otherwise we get a contradiction since $0.8+0.2 \cdot 0.95=0.99$). Thus, when $y \sim D_b$ for such a string $b$, w.p. 0.95, $b'$ will agree with $b$ on at least $0.99$ fraction of coordinates from the above discussion. In another word, for such a string $b$, \[
\E_{y \sim D_b}[\sum_i 1(b_i=b'_i)] \ge 0.95 \cdot 0.99 n.
\] 
This indicates that there are at least $0.65$ fraction of coordinates in $[d]$ that satisfy $\Pr_y[b_i=b'_i] \ge 0.8$ (otherwise we get a contradiction since $0.65+0.35 \cdot 0.8=0.93$). Hence, such a string $b$ will have $0.65$ fraction of \emph{good} coordinates.

Back to the counting of bad edges removed by the 2nd condition, we will remove at most $2 \cdot (0.2+0.8 \cdot 0.35)=0.96$ fraction of edges. Because $1/30+0.96<1$, we know the existence of $b$ and $b^{(i)}$ such that $i$ is a good coordinate in $b$ and $b^{(i)}$.

Now we use $b, b^{(i)}$ and $P$ to construct an algorithm for Bob to win the game in Lemma~\ref{lem:information_lower_dim_1}. From the definition, $P$ makes $60m/d$ queries in expectation on entries $y_{(i-1)(n/d)+1},\ldots,y_{i(n/d)}$ and outputs $b'_i=b_i$ with probability 0.8. Since halting $P$ at $20 \cdot 60m/d$ queries on entries $y_{(i-1)(n/d)+1},\ldots,y_{i(n/d)}$ will reduce the success probability to $0.8-0.05=0.75$. We assume $P$ makes at most $1200m/d$ queries with success probability 0.75. Now we describe Bob's strategy to win the game: 
\begin{enumerate}
\item Randomly sample $b_j \sim \{\pm 1\}$ for $j \neq i$.
\item Simulate the algorithm $P$: each time when $P$ asks the label $y_{\ell}$ for $\ell \in [(j-1)n/d+1,jn/d]$.
\begin{enumerate}
\item If $j=i$, Bob queries the corresponding label from Alice.
\item Otherwise, Bob generates the label using $D_{b_j}$ by himself.
\end{enumerate}
\item Finally, we use $P$ to produce $b' \in \{\pm 1\}^n$ and outputs $b'_i$.
\end{enumerate}
Bob wins this game with probability at least
\[
\Pr_{a}[a=-1] \cdot \Pr_{y}[b'_i=-1] + \Pr_{a}[a=1] \cdot \Pr_{y}[b'_i=1] \ge 0.75
\]
from the properties of $i$ in $b$ and $b^{(i)}$. So we know $1200m/d=\Omega(1/\eps^2)$, which lower bounds $m=\Omega(d/\eps^2)$.
\end{proofof}

\end{document}